%% file: e2e.tex
\documentclass{article}
\title{End-to-end Learning of a Convolutional Neural Network\\via Deep Tensor Decomposition}
\author{}
\input{notation}

\usepackage{multirow}

\usepackage{algorithm}
\usepackage{algpseudocode}
\author{Samet Oymak\thanks{{Department of Electrical and Computer Engineering, University of California, Riverside, CA}}\quad and\quad Mahdi Soltanolkotabi\thanks{Ming Hsieh Department of Electrical Engineering, University of Southern California, Los Angeles, CA}}

\begin{document}
\maketitle
\begin{abstract} 
In this paper we study the problem of learning the weights of a deep convolutional neural network. We consider a network where convolutions are carried out over non-overlapping patches with a single kernel in each layer. We develop an algorithm for simultaneously learning all the kernels from the training data. Our approach dubbed Deep Tensor Decomposition (DeepTD\footnote{Pun intended.}) is based on a rank-$1$ tensor decomposition. We theoretically investigate DeepTD under a realizable model for the training data where the inputs are chosen i.i.d.~from a Gaussian distribution and the labels are generated according to planted convolutional kernels. We show that DeepTD is data--efficient and provably works as soon as the sample size exceeds the total number of convolutional weights in the network. We carry out a variety of numerical experiments to investigate the effectiveness of DeepTD and verify our theoretical findings.
\end{abstract}
\input{intro}
\section{Problem formulation and models}
In this section we discuss the CNN model which is the focus of this paper. Our exposition is directly adapted from the companion paper \cite{first_layer}. A fully connected artificial neural network is composed of computational units called neurons. The neurons are decomposed into layers consisting of one input layer, one output layer and a few hidden layers with the output of each layer is fed in (as input) to the next layer. In a CNN model the output of each layer is related to the input of the next layer by a convolution operation. In this paper we focus on CNN model where the stride length is equal to the length of the kernel. This is sometimes referred to as a non-overlapping convolution operation. We now formally define this operation for future reference. 
\begin{definition}[Non-overlapping convolution]
For two vectors $\vct{k}\in\R^d$ and $\vct{h}\in\R^{p=d\bar{p}}$ their non-overlapping convolution, denoted by $\vct{k}\boxast\vct{h}$ yields a vector $\vct{u}\in\R^{\bar{p}=\frac{p}{d}}$ whose entries are given by
\begin{align*}
\vct{u}_i=\langle \vct{k}^{(\ell)},\vct{h}[i]\rangle\quad\text{where}\quad \vct{h}[i]:=\begin{bmatrix}\vct{h}_{(i-1)d+1}\\\vct{h}_{(i-1)d+2}\\\vdots\\\vct{h}_{id}\end{bmatrix}.
\end{align*}
\end{definition}
As mentioned earlier the non-overlapping convolution resembles the standard convolution between two vectors $\vct{k}\in\R^d$ and $\vct{h}\in\R^{p}$, denoted by $\vct{k}\circledast \vct{h}$, which yields a vector in $\R^{d+p}$. The main difference is that rather than sliding the kernel one entry at a time across overlapping patches of $\vct{h}$ in the convolution summation we use a stride length equal to the size of the kernel (i.e.~slide the kernel by the length of the kernel) so that the summation is carried out over non-overlapping patches of $\vct{h}$. In this paper it is often convenient to view convolutions as matrix/vector multiplications. This leads us to the definition of the kernel matrix below.
\begin{definition}[Kernel matrix]\label{kermat} Consider a kernel $\vct{k}\in\R^d$ and any vector $\vct{h}\in\R^{p=\bar{p}d}$. Corresponding to the non-overlapping convolution $\vct{k}\boxast\vct{h}$, we associate a kernel matrix $\mtx{K}_{\boxast}\in\R^{\bar{p}\times p}$ defined as
\begin{align*}
\mtx{K}_{\boxast}:=\mtx{I}_{\bar{p}}\otimes \vct{k}^T=\begin{bmatrix}\vct{k}^T & \vct{0} & \vct{0}&\ldots &\vct{0}\\\vct{0} &\vct{k}^T &\vct{0}&\ldots &\vct{0}\\\vdots & \vdots &\vdots & \ddots &\vdots \\
\vct{0} & \vct{0} & \vct{0} &\ldots &\vct{k}^T \end{bmatrix}.
\end{align*}
Here, $\mtx{A}\otimes \mtx{B}$ denotes the Kronecker product between the two matrices $\mtx{A}$ and $\mtx{B}$ and $\mtx{I}_{\bar{p}}$ denotes the $\bar{p}\times \bar{p}$ identity matrix. We note that based on this definition $\vct{k}\boxast \vct{h}=\mtx{K}_{\boxast}\vct{h}$. Throughout the paper we shall use $\mtx{K}$ interchangeably with $\mtx{K}_{\boxast}$ to denote this kernel matrix with the dependence on the underlying kernel and its non-overlapping form implied.
\end{definition}
 With the definition of the non-overlapping convolution and the corresponding kernel matrix in hand we are now ready to define the CNN model which is the focus of this paper. For the convenience of the reader the CNN input-output relationship along with the corresponding notation is depicted in Figure \ref{NCNNmodel}.  
  \begin{figure}
\centering
\begin{tikzpicture}
\node at (-2.5,0) {\includegraphics[scale=0.35]{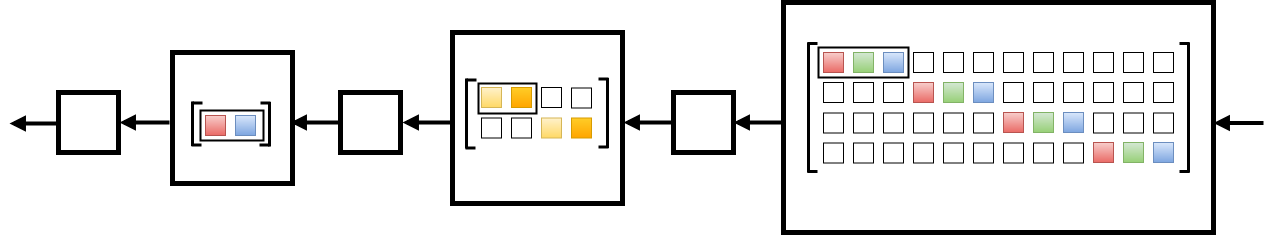}};
\node[black] at (0.6,1.125) {$\vct{k}^{(1)}\in\R^{d_1}$};
\node[black] at (-3.7,0.8) {$\vct{k}^{(2)}\in\R^{d_2}$};
\node[black] at (-7.45,0.55) {\small$\vct{k}^{(3)}\in\R^{d_3}$};
\node[black] at (2,-0.9) {$\vct{K}^{(1)}\in\R^{p_1\times p}$};
\node[black] at (-3.72,-0.6) {${\small\vct{K}^{(2)}\in\R^{p_2\times p_1}}$};
\node[black] at (-7.5,-0.5) {\tiny{$\vct{K}^{(3)}\in\R^{1\times p_2}$}};
\node[black] at (-1.65,-0.05) {\Large{$\phi_1$}};
\node[black] at (-5.75,-0.05) {\Large{$\phi_2$}};
\node[black] at (-9.25,-0.05) {\Large{$\phi_3$}};
\node[black] at (5.2,0.25){\small{$\vct{x}\in\R^p$}};
\node[black] at (5.2,-0.29) {\small{$\vct{h}^{(0)}$}}; 
\node[black] at (-0.975,0.25) {\scriptsize{$\bar{\vct{h}}^{(1)}$}};
\node[black] at (-2.35,0.25) {\scriptsize{$\vct{h}^{(1)}$}};
%\node[black] at (-1.525,-0.35) {\small{$\vct{x}_{\vct{w}}$}}; 
\node[black] at (-5.075,0.25) {\scriptsize{$\bar{\vct{h}}^{(2)}$}};
\node[black] at (-6.45,0.25) {\scriptsize{$\vct{h}^{(2)}$}};
\node[black] at (-10.35,0.25) {\small{$\func{\vct{x}}~$}};
\node[black] at (-10.45,-0.35) {\small{$\vct{h}^{(D)}~~$}};
\node[black] at (-10.35,-0.75) {\small{$=\vct{h}^{(3)}$}};
\node[black] at (-8.55,0.25) {\scriptsize{$\bar{\vct{h}}^{(3)}$}};
\end{tikzpicture}
\caption{Depiction of the input-output relationship of a non-overlapping Convolutional Neural Network (CNN) model along with the various notations and symbols.} %In this example, the output activation $\phi_3$ is chosen to be identity which implies $\func{\x}=\hh^{(D)}=\xbr^{(D)}$.}
\label{NCNNmodel}
\end{figure}
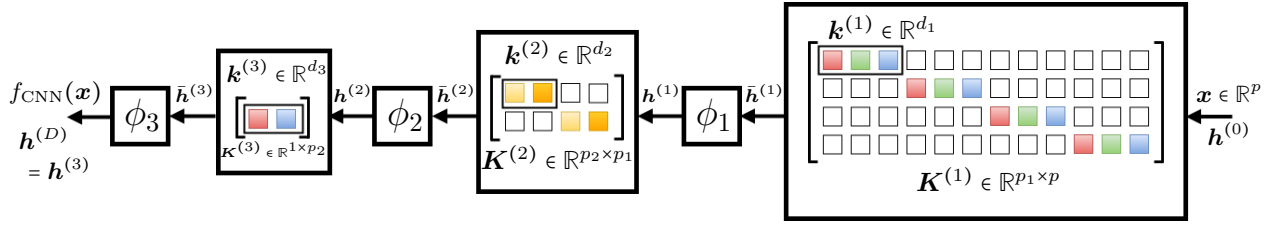
 \begin{itemize}
\item \textbf{Depth and numbering of the layers.} We consider a network of depth $D$ where we number the input as layer $0$ and the output as layer $D$ and the hidden layers $1$ to $D-1$. 
\item \textbf{Layer dimensions and representations.} We assume the input of the CNN, denoted by $\vct{x}\in\R^p$, consists of $p$ features and the output is a one dimensional label. We also assume the hidden layers (numbered by $\ell=1,2,\ldots, D-1$) consists of  $p_\ell$ units with $\bar{\vct{h}}^{(\ell)}\in\R^{p_\ell}$ and $\vct{h}^{(\ell)}\in\R^{p_\ell}$ denoting the input and output values of the units in the $\ell$th hidden layer. For consistency of our notation we shall also define $\vct{h}^{(0)}:=\vct{x}\in\R^p$ and note that the output of the CNN is $\vct{h}^{(D)}\in\R$. Furthermore, $p_0=p$ and $p_D=1$.
\item \textbf{Kernel dimensions and representation.} For $\ell=1,\ldots D$ we assume the kernel relating the output of layer $(\ell-1)$ to the input of layer $\ell$ is of dimension $d_\ell$ and is denoted by $\vct{k}^{(\ell)}\in\R^{d_\ell}$. %We note that based on this notation $d_0=d$ is the length of the input kernel. Since in this paper we are interested in learning the input kernel $\vct{k}^{(0)}\in\R^d$ we shall also use the variable $\vct{w}\in\R^d$ to denote this input kernel.
\item \textbf{Inter-layer relationship.} We assume the inputs of layer $\ell$ (denoted by $\bar{\vct{h}}^{(\ell)}\in\R^{p_{\ell}}$) are related to the outputs of layer $(\ell-1)$ (denoted by $\vct{h}^{(\ell-1)}\in\R^{p_{\ell-1}}$) via a non-overlapping convolution
\begin{align*}
\bar{\vct{h}}^{(\ell)}=\vct{k}^{(\ell)}\boxast \vct{h}^{(\ell-1)}=\mtx{K}^{(\ell)}\vct{h}^{(\ell-1)}\quad\text{for}\quad \el=1,\ldots, D.
\end{align*}
In the latter equality we have used the representation of non-overlapping convolution as a matrix/vector product involving the kernel matrix $\mtx{K}^{(\ell)}\in\R^{p_{\ell}\times p_{\ell-1}}$ associated with the kernel $\vct{k}^{(\ell)}\in\R^{d_\ell}$ per Definition \ref{kermat}. We note that by the nature of the non-overlapping convolution we have that $p_{\ell}=p_{\ell-1}/d_{\ell}$.
\item \textbf{Activation functions and intra-layer relationship.} We assume the input of each hidden unit is related to its output by applying an activation function $\phi_\ell:\R\rightarrow \R$. More precisely, $\vct{h}^{(\ell)}:=\phi_\ell(\bar{\vct{h}}^{(\ell)})$ where for a vector $\vct{u}\in\R^p$, $\phi_\ell(\vct{u})\in\R^p$ is a vector obtained by applying the activation function $\phi_\ell$ to each of the entries of $\vct{u}$. We allow for using distinct activation functions $\{\phi_\el\}_{\el=1}^D$ at every layer. Throughout, we also assume all activations are $1$-Lipschitz functions (i.e.~$\abs{\phi_\ell(a)-\phi_\ell(b)}\le |a-b|$). %focus on regression problems and thus use the identity activation for the output layer i.e.~$\vct{h}^{(D)}=\bar{\vct{h}}^{(D)}$. 
\item \textbf{Final output.} To summarize the input-output relation of our CNN model with an input $\vct{x}\in\R^p$ and input kernel $\vct{k}^{(1)}\in\R^{d_1}$ is given by
\begin{align}
\label{CNNform}
\vct{x}\mapsto f_{CNN}(\vct{x}):=\hh^{(D)},
\end{align}
with
\begin{align}
\label{hiddenform}
&\hh^{(\el)}=\phi_\ell(\xbr^{(\el)})\quad \text{and}\quad \xbr^{(\el)}=\Kb^{(\ell)}\hh^{(\el-1)}=\Kb^{(\ell)}\phi_{\ell-1}\left(\Kb^{(\ell-1)}\phi_{\ell-2}\left(\dots\left(\Kb^{(2)}\phi_1\left(\mtx{K}^{(1)}\vct{x}\right)\right)\right)\right).
\end{align}
\end{itemize}
\section{Algorithm: Deep Tensor Decomposition (DeepTD)}
This paper introduces an approach to approximating the convolutional kernels from training data based on tensor decompositions dubbed DeepTD, which consists of a carefully designed tensor decomposition. To connect these two problems, we begin by stating how we intend to construct the tensor from the training data. To this aim given any input data $\x\in\R^p$, we form a $D$-way tensor $\X\in\R^{\btd}$ as follows. First, we convert $\x$ into a matrix by placing every $d_1$ consecutive entries of $\vct{x}$ as a row of a matrix of size $p_1\times d_1$. From this matrix we then create a 3-way tensor of size $p_2\times d_2\times d_1$ by grouping $d_2$ consecutive entries of each of the $d_1$ columns and so on. We repeat this procedure $D$ times to arrive at the $D$-way tensor $\X\in\R^{\btd}$. We define $\mathcal{T}:\R^p\mapsto \R^{\btd}$ as the corresponding  tensor operation that maps $\vct{x}\in\R^p$ to $\mtx{X}\in\R^{\btd}$.

Given a set of training data consisting of $n$ input/output pairs $(\vct{x}_i,y_i)\in\R^p\times \R$ we construct a tensor $\mtx{T}_n$ by tensorizing the input vectors as discussed above and calculating a weighted combination of these tensorized inputs where the weights depends on the output labels fo the training data. More precisely,
\begin{align}
\label{defTn}
\mtx{T}_n:=\frac{1}{n}\sum_{i=1}^n \left(y_i-y_{avg}\right)\mtx{X}_i\quad \text{where}\quad y_{avg}=\frac{1}{n}\sum_{i=1}^ny_i\quad\text{and}\quad \mtx{X}_i=\mathcal{T}(\vct{x}_i).
\end{align}
We then perform a rank-$1$ tensor decomposition on this $T$ tensor to approximate the convolutional kernels. Specifically we solve
\begin{align}
\label{deeptdalg}
\hat{\vct{k}}^{(1)},\ldots,\hat{\vct{k}}^{(D)}=\underset{\vct{v}_1\in\R^{d_1},\vct{v}_2\in\R^{d_2},\ldots,\vct{v}_D\in\R^{d_D}}{\arg\max} \li\Tb_n,\robt{\vb}\ri\quad\text{subject to}\quad \twonorm{\vct{v}_1}=\twonorm{\vct{v}_2}=\ldots=\twonorm{\vct{v}_D}=1.
\end{align}
In the above $\robt{\vb}$ denotes the tensor resulting from the outer product of the vectors $\vct{v}_1,\vct{v}_2,\ldots,\vct{v}_D$. This tensor rank decomposition is also known as CANDECOMP/PARAFAC (CP) decomposition \cite{bro1997parafac} and can be solved efficiently using Alternating Least Squares (ALS) and a variety of other algorithms \cite{anandkumar2014tensor,ge2015escaping,anandkumar2014guaranteed}.\footnote{We would like to note that we believe that the assumptions under which the above stated algorithms solve the problem \eqref{deeptdalg} holds for the empirical tensor $\mtx{T}_n$ under our planted model. We do not however pursue this here and leave it to future work.}

At this point it is completely unclear why the tensor $T_n$ or its rank-$1$ decomposition can yield anything useful. The main intuition is that as the data set grows ($n\rightarrow\infty$) the \emph{empirical} tensor $\Tb_n$ converges close to a \emph{population} tensor $\mtx{T}$ whose rank-$1$ decomposition reveals useful information about the factors. Specifically, we will show that
\begin{align}
\label{approxiden}
%\frac{n}{n-1}\E[\mtx{T}_n]=\Tb:=\E_{\vct{x}\sim\mathcal{N}(\vct{0},\mtx{I}_p)}[\func{\vct{x}}\mathcal{T}(\vct{x})]\approx \alpha \robtu{\vct{k}},
\underset{n\rightarrow\infty}{\lim}\mtx{T}_n=\Tb:=\E_{\vct{x}\sim\mathcal{N}(\vct{0},\mtx{I}_p)}[\func{\vct{x}}\mathcal{T}(\vct{x})]\approx \alpha \robtu{\vct{k}},
\end{align}
with $\alpha$ a scalar whose value shall be discussed later on. Here, $\vct{x}$ is a Gaussian random vector with i.i.d.~$\mathcal{N}(0,1)$ entries and represents a typical input with $\func{\vct{x}}$ the corresponding output and $\mathcal{T}(\vct{x})$ the tensorized input. We will also utilize a concentration argument to show that when the training data set originates from an i.i.d.~distribution, for a sufficiently large training data $n$, $\Tb_n$ yields a good approximation of the population tensor $\Tb$.

Another perhaps perplexing aspect of the construction of $\Tb_n$ in \eqref{defTn} the subtraction by $y_{avg}$ in the weights. The reason this may be a source of confusion is that based on the intuition above
\begin{align*}
\E[\Tb]=\E\Big[\frac{1}{n}\sum_{i=1}^n y_i\mtx{X}_i\Big]=\E\Big[\frac{1}{n-1}\sum_{i=1}^n \left(y_i-y_{avg}\right)\mtx{X}_i\Big]=\frac{n}{n-1}\E[\Tb_n]\approx \alpha \robtu{\vct{k}},
\end{align*}
so that the subtraction by the average seems completely redundant. The main purpose of this subtraction is to ensure the weights $y_i-y_{avg}$ are centered (have mean zero). This centering allows for a much better concentration of the empirical tensor around its population counter part and is crucial to the success of our approach. We would like to point out that such a centering procedure is reminiscent of batch-normalization heuristics deployed when training deep neural networks.

Finally, we note that based on \eqref{approxiden}, the rank-$1$ tensor decomposition step can recover the convolutional kernels $\{\lay{\el}\}_{\el=1}^D$ up to sign and scaling ambiguities. Unfortunately, depending on the activation function, it may be impossible to overcome these ambiguities. For instance, if the activations are homogeneous (i.e. $\phi_\ell(ax)=a\phi_\ell(x)$), then scaling up one layer and scaling down the other layer by the same amount does not change the overall function $\func{\cdot}$. Similarly, if the activations are odd functions, negating two of the layers at the same time preserves the overall function. In Section \ref{resolve} we discuss some heuristics ands theoretical guarantees for overcoming these sign/scale ambiguities.

\section{Main results}\label{sec main res}
In this section we introduce our theoretical results for DeepTD. We will discuss these results in three sections. In Section \ref{sec1} we show that the empirical tensor concentrates around its population counterpart. Then in Section \ref{sec2} we show that the population tensor is well-approximated by a rank-1 tensor whose factors reveal the convolutional kernels. Finally, in Section \ref{sec3} we combine these results to show DeepTD can approximately learn the convolutional kernels up to sign/scale ambiguities.

\subsection{Concentration of the empirical tensor}
\label{sec1}
Our first result shows that the empirical tensor concentrations around the population tensor. We measure the quality of this concentration via the tensor spectral norm defined below.
\begin{definition}
Let $\Ro$ be the set of rank-one tensors characterized by
\[
\Ro=\bigg\{\mtx{V}\in \R^{\btd}\bgl\mtx{V}=\robt{\vb},~\tn{\vb_i}\leq1~\text{for all}~1\leq i\leq D\bigg\}.
\]
The spectral norm of a tensor $\X\in\R^{\btd}$ is given by the supremum,
\[
\tsn{\Tb}=\sup_{\mtx{V}\in \Ro} \li\mtx{V},\Tb\ri.
\]

\end{definition}
%
%We first provide a finite sample analysis of the gap between empirical and population tensors. This bound is in terms of the tensor spectral norm and will be utilized to show robustness of $\nntd$ via Lemma \ref{nntd robust}. To state our result, we introduce $\lipp$ which is an upper bound on the {\em{Lipschitz constant}} of the network,
%\[
%\lipp=\prod_{\el=1}^D\tn{\lay{\el}}.
%\]
%With this definition, our finite sample bound is a direct conclusion of a more general result Theorem \ref{thm dev}. We use the fact that $\func{\cdot}$ is $\lipp$ Lipschitz function of its input $\x$.
\begin{theorem} \label{finite sample} Consider a CNN model $\vct{x}\mapsto f_{CNN}(\vct{x})$ of the form \eqref{CNNform} consisting of $D\ge 2$ layers with convolutional kernels $\vct{k}^{(1)}, \vct{k}^{(2)}, \ldots, \vct{k}^{(D)}$ of lengths $d_1, d_2,\ldots, d_D$. Let $\vct{x}\in\R^p$ be a Gaussian random vector distributed as $\mathcal{N}(\vct{0},\mtx{I}_p)$ with the corresponding labels $y=f_{CNN}(\vct{x})$ generated by the CNN model and $\mtx{X}:=\mathcal{T}(\vct{x})$ the corresponding tensorized input. Suppose the data set consists of $n$ training samples where the feature vectors $\vct{x}_i\in\R^p$ are distributed i.i.d.~$\mathcal{N}(\vct{0},\mtx{I}_p)$ with the corresponding labels $y_i=f_{CNN}(\vct{x}_i)$ generated by the same CNN model and $\mtx{X}_i:=\mathcal{T}(\vct{x}_i)$ the corresponding tensorized input. Then the empirical tensor $\mtx{T}_n$ and population tensor $\mtx{T}$ defined based on this dataset obey
\begin{align}
\label{concenguaran}
\tsn{\mtx{T}_n-\mtx{T}}:=\tsn{\frac{1}{n}\sum_{i=1}^n(y_i-\yavg)\X_i-\E[y\X]}\leq c\lipp\frac{\sqrt{\left(\sum_{\ell=1}^D d_\ell\right)\log D}+t}{\sqrt{n}},
\end{align}
with probability at least $1-5e^{-\min\left(t^2,t\sqrt{n},n\right)}$, where $c>0$ is an absolute constant.
\end{theorem}
The theorem above shows that the empirical tensor approximates the population tensor with high probability. This theorem also shows that the quality of this approximation is proportional to $\lipp$. This is natural as $\lipp$ is an upper-bound on the Lipschitz constant of the network and shows how much the CNN output fluctuates with changes in the input. The more fluctuations, the less concentrated the empirical tensor is, translating into a worse approximation guarantee. Furthermore, the quality of this approximation grows with the square root of the parameters in the model ($\sum_{\ell=1}^D d_\ell$) and is inversely proportional to the square root of the number of samples ($n$) which are typical scalings in statistical learning. We would also like to note that as we will see in the forthcoming sections, in many cases $\tsn{\mtx{T}}$ is roughly on the order of $\lipp$ so that \eqref{concenguaran} guarantees that $\tsn{\mtx{T}_n-\mtx{T}}/\tsn{\mtx{T}}\le c'\sqrt{\frac{\left(\sum_{\ell=1}^D d_\ell\right)\log D}{n}}$. Therefore, the relative error in the approximation is less than $\epsilon$ as soon as the number of observations exceeds the number of parameters in the model by a logarithmic factor in depth i.e.~$n\gtrsim (\sum_{\ell=1}^D d_\ell)\frac{\log D}{\epsilon^2}$.
\subsection{Rank one approximation of the population tensor}
\label{sec2}
Our second result shows that the population tensor can be approximated by a rank one tensor. To explain the structure of this rank one tensor and quantify the quality of this approximation we require a few definitions. The first quantity roughly captures the average amount by which the nonlinear activations amplify or attenuate the size of an input feature at the output. This quantity is formally defined below.
\begin{definition}[CNN gain] Let $\x\sim\Nn(0,\Iden_p)$ and define the hidden unit/output values of the CNN based on this random input per equations \eqref{hiddenform} and \eqref{CNNform}. We define the CNN gain as
\[
\cgain=\prod_{\el=1}^D\E[\sigmap_\el(\xbr^{(\el)}_1)].
\]
In words, this is the product of expectations of the activations evaluated at the first entry of each layer.
\end{definition}
This quantity is the product of the average slopes of the activations evaluated along a path connecting the first input feature to the first hidden units across the layers all the way to the output. We note that this quantity is the same when calculated along any path connecting an input feature to the output passing through the hidden units. Therefore, this quantity can be thought of as the average gain (amplification or attenuation) of a given input feature due to the nonlinear activations in the network. To gain some intuition consider a ReLU network which is mostly inactive. Then the network is dead and $\cgain\approx0$. On the other extreme if all ReLU units are active the network operates in the linear regime and $\cgain=1$. We would like to point out that $\cgain$ can in many cases be bounded from below by a constant. For instance for ReLU activations as long as the kernels obey
\begin{align}
\label{diffusenesstt}
\left(\vct{1}^T\vct{k}^{(\ell)}\right)\ge 4\twonorm{\vct{k}^{(\ell)}},
\end{align}
then $\cgain\ge \gamma$ with $\gamma$ a fixed numerical constant (formally proved in our companion paper \cite{first_layer}). We note that an assumption similar to \eqref{diffusenesstt} is needed for the network to be active. This is because if the kernel sums are negative one can show that with high probability, all the ReLUs after the first layer will be inactive and the network will be dead. In \cite{first_layer} we also show that similar results hold for other popular activations including softplus. With this definition in hand, we are now ready to describe the form of the rank one tensor that approximates the population tensor.
%Observe that for increasing activations obeying $\sigmap_\el\geq \beta>0$, we immediately have $\cgain\geq\beta^D$. We will show that population tensor $\E[y\X]$ is close to the rank $1$ CNN tensor defined as follows.
\begin{definition} [Rank one CNN tensor] \label{tcnn def}We define the rank one CNN tensor $\lcnn\in\R^{\btd}$ as
\[
\lcnn=\cgain \bt_{\el=1}^D\lay{\el}.
\]
In words, this is the product of the kernels $\{\lay{\el}\}_{\el=1}^D$ scaled by the CNN gain $\cgain$.
\end{definition}
To quantify how well the rank one CNN tensor approximates the population tensor we need two definitions. The first definition concerns the activation functions.
\begin{definition}[Activation smoothness] \label{assume smooth} We assume the activations are differentiable everywhere and $S$-smooth (i.e.~$|\sigmap_\el(x)-\sigmap_\el(y)|\leq \smo|x-y|$ for all $x,y\in\R$) for some $\smo\geq 0$.
\end{definition}
The reason smoothness of the activations play a role in the quality of the rank one approximation is that smoother activations translate into smoother variations in the entries of the population tensor. Therefore, the population tensor can be better approximated by a low-rank tensor. The second definition captures how diffused the kernels are.
\begin{definition}[Kernel diffuseness parameter] \label{assume diff} Given kernels $\{\lay{\el}\}_{\el=1}^D$ with dimensions $\{d_\ell\}_{\ell=1}^D$, the kernel diffuseness parameter $\mu$ is defined as
\[
\mapp=\sup_{1\leq \el\leq D}\frac{\sqrt{d_\ell}\infnorm{\lay{\el}}}{\twonorm{\lay{\el}}}.%\prod_{i=1}^{\el-1}{\tn{\lay{i}}}.
\]

\end{definition}
%The distance between the population tensor and $\lcnn$ will be upper bounded by the incoherence parameter introduced below.
%\begin{lemma} MAI satisfies $\mapp\leq \mupp\mipp$ where $\mupp$ is the maximum incoherence and $\mipp$ is the maximum Lipschitzness that are defined as%There exists $\gamma>0$ so that every layer obeys $\|\lay{\el}\|_\infty\leq \gamma % \sqrt{d_\el}
%\[
%\mupp=\sup_{1\leq \el\leq D}\frac{\|\lay{\el}\|_\infty}{\tn{\lay{\el}}},~~~\mipp=\sup_{1\leq \el\leq D}\prod_{i=1}^\el\tn{\lay{i}}.
%\]
%\end{lemma}
%\begin{proof} The proof follows from writing ${\|\lay{\el}\|_\infty}\prod_{i=1}^{\el-1}{\tn{\lay{i}}}=\frac{\|\lay{\el}\|_\infty}{\tn{\lay{\el}}}\prod_{i=1}^{\el}{\tn{\lay{i}}}\leq \mupp\mipp$.
%\end{proof}
The less diffused (or more spiky) the kernels are, the more the population tensor fluctuates and thus the quality of the approximation to a rank one tensor decreases. With these definitions in place, we are now ready to state our theorem on approximating a population tensor with a rank one tensor.
\begin{theorem}\label{thm pop} Consider the setup of Theorem \ref{finite sample}. Furthermore, assume the activations are $S$-smooth per Definition \ref{assume smooth} and the convolutional kernels are $\mu$-diffused per Definition \ref{assume diff}. Then, the population tensor $\mtx{T}:=\E[y\X]$ can be approximated by the rank-$1$ tensor $\lcnn:=\cgain \bt_{\el=1}^D\lay{\el}$ as follows
\[
\tsn{{\mtx{T}-\lcnn}}\leq \tf{\mtx{T}-\lcnn}\leq \sqrt{8\pi}\mu\smo\cdot \prod_{i=1}^D \twonorm{\vct{k}^{(i)}}\cdot\underset{\ell}{\sup} \prod_{i=1}^\ell \twonorm{\vct{k}^{(i)}}\frac{D}{\sqrt{\underset{ \ell }{\min}\text{ }d_\ell}}.
\]
%Furthermore, applying Lemma \ref{nntd robust}, \nntd~estimate satisfies
%\[
%\frac{\prod_{\el=1}^D|\li\lay{\el},\lah{\el}\ri|}{\lipp}\geq 1-\frac{\sqrt{32\pi}\smo D\mapp}{\cgain}
%\]
%Observe that the left hand side is product of layerwise correlations and is always upper bounded by $1$.
\end{theorem}
The theorem above states that the quality of the rank one approximation deteriorates with increase in the smoothness of the activations and the diffuseness of the convolutional kernels. As mentioned earlier increase in these parameters leads to more fluctuations in the population tensor making it less likely that it can be well approximated by a rank one tensor. We also note that $\tsn{\lcnn}=\cgain\prod_{\ell=1}^D \twonorm{\vct{k}^{(\ell)}}$ and therefore the relative error in this approximation is bounded by
\[
\frac{\tsn{\mtx{T}-\lcnn}}{\tsn{\lcnn}}\le \sqrt{8\pi}\frac{\mu\smo}{\cgain}\underset{\ell}{\sup} \prod_{i=1}^\ell \twonorm{\vct{k}^{(i)}}\frac{ D}{\sqrt{\underset{ \ell }{\min}\text{ }d_\ell}}.
\]%in some cases
We would like to note that for many activations the smoothness is bounded by a constant. For instance, for the softplus activation ($\phi(x)=\log(1+e^x)$) and one can show that $S\le 1$. As stated earlier, under appropriate assumptions on the kernels and activations, the CNN gain $\cgain$ is also bounded from below by a constant. Assuming the convolutional kernels have unit norm and are sufficiently diffused so that the diffuseness parameter is bounded by a constant we thus arrive at the conclusion
\[
\frac{\tsn{\mtx{T}-\lcnn}}{\tsn{\lcnn}}\le \frac{c}{\cgain}\frac{D}{\sqrt{\underset{ \ell }{\min}\text{ }d_\ell}}.
\]
This implies that as soon as the length of the convolutional patches scale with the square of depth of the network by a constant factor the rank one approximation is sufficiently good. Our back-of-the-envelope calculations suggest that the correct scaling is linear in $D$ versus the quadratic result we have established here. Improving our result to achieve the correct scaling is an interesting future research direction. Finally, we would like to note that while we have assumed differentiable and smooth activations we expect our results to apply to popular non-differentiable activations such as ReLU activations. Indeed, as it will become clear in our proofs (e.g. see Lemma \ref{xyz}) an average notion of smoothness appears to be sufficient for our results to apply.

\subsection{Learning the convolutional kernels}
\label{sec3}
We demonstrated in the previous two sections that the empirical tensor concentrates around its population counter part and that the population tensor is well-approximated by a rank one tensor. We combine these two results along with a perturbation argument to provide guarantees for the DeepTD algorithm.
\begin{theorem}[Main theorem]\label{main thm} Consider a CNN model $\vct{x}\mapsto f_{CNN}(\vct{x})$ of the form \eqref{CNNform} consisting of $D\ge 2$ layers with convolutional kernels $\vct{k}^{(1)}, \vct{k}^{(2)}, \ldots, \vct{k}^{(D)}$ of lengths $d_1, d_2,\ldots, d_D$. Let $\vct{x}\in\R^p$ be a Gaussian random vector distributed as $\mathcal{N}(\vct{0},\mtx{I}_p)$ with the corresponding labels $y=f_{CNN}(\vct{x})$ generated by the CNN model and $\mtx{X}:=\mathcal{T}(\vct{x})$ the corresponding tensorized input. Suppose the data set consists of $n$ training samples where the feature vectors $\vct{x}_i\in\R^p$ are distributed i.i.d.~$\mathcal{N}(\vct{0},\mtx{I}_p)$ with the corresponding labels $y_i=f_{CNN}(\vct{x}_i)$ generated by the same CNN model and $\mtx{X}_i:=\mathcal{T}(\vct{x}_i)$ the corresponding tensorized input. Also, assume the activations are $1$-Lipschitz (i.e.~$\abs{\phi_\ell'(z)}\le 1$), are $S$-smooth per Definition \ref{assume smooth} and the convolutional kernels are $\mu$-diffused per Definition \ref{assume diff}. Then the empirical tensor $\mtx{T}_n:=\frac{1}{n}\sum_{i=1}^n(y_i-\yavg)\X_i$ and the rank one approximation $\lcnn=\cgain \bt_{\el=1}^D\lay{\el}$ to the population tensor $\mtx{T}=\E[y\mtx{X}]$ obey
\[
\tsn{\mtx{T}_n-\lcnn}\leq \lipp\left(c\frac{\sqrt{\left(\sum_{\ell=1}^D d_\ell\right)\log D}+t}{\sqrt{n}}+\sqrt{8\pi}\mu\smo\cdot\underset{\ell}{\sup} \prod_{i=1}^\ell \twonorm{\vct{k}^{(i)}}\frac{D}{\sqrt{\underset{ \ell }{\min}\text{ }d_\ell}}\right),
\]
with probability at least $1-5e^{-\min\left(t^2,t\sqrt{n},n\right)}$, where $c>0$ is an absolute constant.
Furthermore, the \nntd~estimates of the convolutional kernels given by \eqref{deeptdalg} using the empirical tensor $\mtx{T}_n$ obeys
\[
\frac{\prod_{\el=1}^D|\li\lay{\el},\lah{\el}\ri|}{\lipp}\geq 1-\frac{2}{\cgain}\left(c\frac{\sqrt{\left(\sum_{\ell=1}^D d_\ell\right)\log D}+t}{\sqrt{n}}+\sqrt{8\pi}\mu\smo\cdot\underset{\ell}{\sup} \prod_{i=1}^\ell \twonorm{\vct{k}^{(i)}}\frac{D}{\sqrt{\underset{ \ell }{\min}\text{ }d_\ell}}\right),
\]
with probability at least $1-5e^{-\min\left(t^2,t\sqrt{n},n\right)}$, where $c>0$ is an absolute constant.
\end{theorem}
The above theorem is our main result on learning a non-overlapping CNN with a single kernel at each layer. It demonstrates that estimates $\lah{\el}$ obtained by DeepTD have significant inner product with the ground truth kernels $\lay{\el}$ with high probability, using only few samples. Indeed, similar to the discussion after Theorem \ref{thm pop} assuming the activations are sufficiently smooth and the convolutional kernels are unit norm and sufficiently diffused, the theorem above can be simplified as follows
\begin{align*}
\frac{\prod_{\el=1}^D|\li\lay{\el},\lah{\el}\ri|}{\lipp}\ge 1-c \left(\frac{\sqrt{\left(\sum_{\ell=1}^D d_\ell\right)\log D}}{\sqrt{n}}+\frac{D}{\sqrt{\underset{ \ell }{\min}\text{ }d_\ell}}\right).
\end{align*}
Thus the kernel estimates obtained via DeepTD are well aligned with the true kernels as soon as the number of samples scales with the total number of parameters in the model and the length of the convolutional kernels (i.e.~the size of the batches) scales quadratically with the depth of the network.
\input{cerm}

\input{numerical}

\section{Related work}
Our work is closely related to the recent line of papers on neural networks as well as tensor decompositions. We briefly discuss this related literature.

\noindent{\bf{Neural networks:}} Learning neural networks is a nontrivial task involving non-linearities and non-convexities. Consequently, existing theory works consider different algorithms, network structures and assumptions. A series of recent work focus on learning zero or one-hidden layer fully connected neural networks with random inputs and planted models \cite{soltanolkotabi2017learning, WF, soltanolkotabi2017structured, oymak2018learning, mei2018mean, zhong2017recovery, ge2017learning, fu2018local,jaganathan2012recovery}. Other recent publications \cite{brutzkus2017globally, du2017gradient, zhong2017learning,du2017convolutional,oymak2018learning} consider the problem of learning a convolutional neural network with $1$-hidden layer. In particular, \cite{brutzkus2017globally,du2017gradient,du2017convolutional} focuses on learning non-overlapping networks as in this paper (albeit in the limit of infinite training data). The papers above either focus on characterizing the optimization landscape, or population landscape or providing exact convergence guarantees for gradient descent. In comparison, in this paper we focus on approximate convergence guarantees using tensor decompositions for arbitrary deep networks. A few recent publications \cite{sol2017,sagun2017empirical,soudry2016no} consider the training problem when the network is over-parametrized and study the over-fitting ability of such networks. Closer to our work, \cite{malach2018provably} and \cite{arora2014provable} considers provable algorithms for deep networks using layer-wise algorithms. In comparison to our work, \cite{malach2018provably} applies to a very specific generative model that assumes a discrete data distribution and studies the population loss in lieu of the empirical loss. Arora et al. \cite{arora2014provable} studies deep models but uses activation functions that are not commonly used and assumes random network weights. In comparison, we work with practical activations and do not assume the weights are generated at random.

%du2018power,
%\noindent{\bf{Tensor methods:}}
%
%Among these, \cite{zhong2017recovery, janzamin2015beating} develop tensor methods 
%
%
%\cite{sol2017,du2017gradient,li2017convergence,}
%Tensor decomposition itself is a highly non-convex and nontrivial task. 

\noindent {\bf{Tensor decomposition:}} Tensors are powerful tools to model a wide variety of big-data problems \cite{sidiropoulos2017tensor,anandkumar2014tensor}. Recent years have witnessed a growing interest in tensor decomposition techniques to extract useful latent information from the data \cite{anandkumar2014tensor,ge2015escaping}. The connection between tensors and neural networks have been noticed by several papers \cite{mondelli2018connection,janzamin2015beating,cohen2016expressive,kossaifi2017tensor,cohen2016convolutional}. Cohen et al. \cite{cohen2016convolutional,cohen2016expressive} relate convolutional neural networks and tensor decompositions to provide insights on the expressivity of CNNs. Mondelli and Montanari \cite{mondelli2018connection} connects the hardness of learning shallow networks to tensor decomposition. Closer to this paper, Janzamin et al. \cite{janzamin2015beating} and Zhong et al. \cite{zhong2017recovery} apply tensor decomposition on $one$-hidden layer, fully connected networks to approximately learn the latent weight matrices. 

\section{Conclusion}
In this paper we studied a multilayer CNN model with depth $D$. We assumed a non-overlapping structure where each layer has a single convolutional kernel and has stride length equal to the dimension of its kernel. We establish a connection between approximating the CNN kernels and higher order tensor decompositions. Based on this connection we proposed an algorithm for simultaneously learning all the kernels called the Deep Tensor Decomposition (DeepTD) algorithm. This algorithm builds a $D$-way tensor based on the training data and applies a rank one tensor factorization algorithm to this tensor to simultaneously estimate all the convolutional kernels. Assuming the input data is distributed i.i.d.~according to a Gaussian model with corresponding output generated by a planted set of convolutional kernels we prove DeepTD can approximately learn all the kernels with a near minimal number of training data. A variety of numerical experiments complement our theoretical findings.

%a planted model 
%We demonstrated that the factors of a rank one approximation of this tensor approximate the kernel estimates

%This corresponds to a non-overlapping structure depicted in Figure \ref{NCNNmodel}. To establish a connection between CNNs and higher order tensors, in Section \ref{ten decomp}, we propose the \nntd~algorithm based on tensor decomposition. This algorithm first casts input data points into $D$-way tensors, and then creates a weighted linear combination of these tensors using {\em{centered}} output labels. These two operations return the CNN tensor $\tcnn$.
%
%\noindent{\bf{Population landscape:}} Investigating population landscape of CNNs, we derive the conditions under which $\E[\tcnn]$ can be rigorously approximated as a rank-$1$ tensor. We introduce a measure of ``network incoherence'' which can be used to bound the rank-$1$ approximation error. We rigorously show that as kernel sizes increase and kernels become diffused (i.e.~all entries have similar magnitude) $\E[\tcnn]$ becomes closer to a rank-$1$ tensor. 
%
%\noindent{\bf{Learning from finite data:}} We complement our findings on the population landscape with finite sample analysis. We find that the spectral norm gap between empirical and population tensors $\tcnn-\E[\tcnn]$ becomes manageable as soon as sample size grows proportional to total number of network parameters. Together with results on population landscape, this implies that, rank-$1$ tensor decomposition can provably learn kernels of each layer up to sign and scaling ambiguity.

\section{Proofs}
In this section we will prove our main results. Throughout, for a random variable $X$, we use $\zm{X}$ to denote $X-\E[X]$. Simply stated, $\zm{X}$ is the {\em{centered}} version of $X$. For a random vector/matrix/tensor $\mtx{X}$, $\zm{\mtx{X}}$ denotes the vector/matrix/tensor obtained by applying the $\zm{}$ operation to each entry. For a tensor $\mtx{T}$ we use $\fronorm{\mtx{T}}$ to denote the square root of the sum of squares of the entries of the tensor. Stated differently, this is the Euclidean norm of a vector obtained by rearranging the entries of the tensor. Throughout we use $c$, $c_1$, $c_2$, and $C$ to denote fixed numerical constants whose values may change from line to line. We begin with some useful definitions and lemmas.
\subsection{Useful concentration lemmas and definitions}
In this section we gather some useful definitions and well-known lemmas that will be used frequently throughout our concentration arguments.
\begin{definition}[Orlicz norms] \label{ornorm}For a scalar random variable Orlicz-$a$ norm is defined as
\[
\|X\|_{\psi_{a}}=\sup_{k\geq 1}k^{-1/a}(\E[|X|^k])^{1/k}
\]
Orlicz-$a$ norm of a vector $\x\in\R^p$ is defined as $\|\x\|_{\psi_{a}}=\sup_{\vb\in \Bc^{p}} \|\vb^T\x\|_{\psi_{a}}$ where $\Bc^p$ is the unit $\el_2$ ball.
The sub-exponential norm is the function $\te{\cdot}$ and the sub-gaussian norm the function $\tsub{\cdot}$.
\end{definition}
We now state a few well-known results that we will use throughout the proofs. This results are standard and are stated for the sake of completeness. The first lemma states that the product of sub-gaussian random variables are sub-exponential.
\begin{lemma} \label{exp obvious}Let $X,Y$ be subgaussian random variables. Then $\te{XY}\leq \tsub{X}\tsub{Y}$.
\end{lemma}
The next lemma connects Orlicz norms of sum of random variables to the sum of the Orlicz norm of each random variable.
\begin{lemma}\label{subexp zm} Suppose $X,Y$ are random variables with bounded $\tsup{\cdot}$ norm. Then $\tsup{X+Y}\leq 2\max\{\tsup{X},\tsup{Y}\}$. In particular $\tsup{X-\E X]}\leq 2\tsup{X}$.
\end{lemma}
The lemma below can be easily obtained by combining the previous two lemmas.
\begin{lemma} \label{exp zm mult}Let $X,Y$ be subgaussian random variables. Then $\te{\zm{XY}}\leq 2\tsub{X}\tsub{Y}$.
\end{lemma}
Finally, we need a few standard chaining definitions.
%o understand where perturbed width arises from, we introduce Talagrand's $\gamma_a$ functionals and associated helper definitions.
\begin{definition} [Admissible sequence \cite{talagrand2014gaussian}] Given a set $T$ an admissible sequence is an increasing sequence $\{A_n\}_{n=0}^\infty$ of partitions of $T$ such that $|A_n| \leq N_n$ where $N_0=1$ and $N_n=2^{2^n}$ for all $n\geq 1$.
\end{definition}
For the following discussion $\Delta_d(A_n(t))$, will be the diameter of the set $S\in A_n$ that contains $t$, with respect to the $d$ metric.
\begin{definition} [$\gamma_a$ functional \cite{talagrand2014gaussian}] \label{gamma functional}Given $a > 0$, and a metric space $(T,d)$ we define \[\gamma_a(T,d) = \inf \sup_{t\in T} \sum_{n\geq 0} 2^{n/a} \Delta_d(A_n(t)),\] where the infimum is taken over all admissible sequences.
\end{definition}

The following lemma upper bounds $\gamma_\alpha$ functional with covering numbers of $T$. The reader is referred to Section $1.2$ of \cite{talagrand2006generic}, Equation $(2.3)$ of \cite{dirksen2013tail}, and Lemma $D.17$ of \cite{oymak2018learning}.
\begin{lemma} [Dudley's entropy integral]\label{dudley lem}Let $N(\eps)$ be the $\eps$ covering number of the set $T$ with respect to the $d$ metric. Then
\[
\gamma_\alpha(T,d)\leq C_\alpha \int_{0}^\infty \log^{1/\alpha} N(\eps) d\eps,
\]
where $C_\alpha>0$ depends only on $\alpha>0$.
\end{lemma}

\subsection{Concentration of the empirical tensor (Proof of Theorem \ref{finite sample})}
%Given $D$ vectors $\{\vb_i\}_{i=1}^D$, their rank $1$ $D$-fold tensorization will be denoted by $\vbb_{1:D}=(\vb_1,\dots,\vb_D)$. 
To prove this theorem, first note that given labels $\{y_i\}_{i=1}^n\sim y$ and their empirical average $\yavg=n^{-1}\sum_{i=1}^ny_i$, we have $\E[\yavg]=\E[y]$. Hence $y-\yavg=\zm{y-\yavg}$ and we can rewrite the empirical tensor as follows
%given i.i.d.~Gaussian training data $(\X_i,y_i)_{i=1}^n$ as described in Theorem \ref{main thm}, observe that
\begin{align}\label{decompose}
\mtx{T}_n:=&\frac{1}{n}\sum_{i=1}^n (y_i-\yavg)\X_i\nonumber\\
=&\frac{1}{n}\sum_{i=1}^n \zm{y_i-\yavg}\X_i\nonumber\\
=&\frac{1}{n}\sum_{i=1}^n \zm{y_i}\X_i-\zm{\yavg}\left(\frac{1}{n}\sum_{i=1}^n\X_i\right).%\\
%&:=\Tb_{\text{main}}+\Tb_{\text{side}}
\end{align}
Recall that the population tensor is equal to $\mtx{T}:=\E[y_i\mtx{X}_i]$. Furthermore, $\E[\E[y_i]\mtx{X}_i]=\E[y_i]\E[\mtx{X}_i]=0$. Thus the population tensor can alternatively be written as $\mtx{T}=E[\zm{y_i}\mtx{X}_i]=\frac{1}{n}\sum_{i=1}^n\E[\zm{y_i}\mtx{X}_i]$. Combining the latter with \eqref{decompose} we conclude that
\begin{align*}
\mtx{T}_n-\mtx{T}=&\frac{1}{n}\sum_{i=1}^n \left(\zm{y_i}\X_i-\E[\zm{y_i}\mtx{X}_i]\right)-\zm{\yavg}\left(\frac{1}{n}\sum_{i=1}^n\X_i\right),\\
=&\frac{1}{n}\sum_{i=1}^n\zm{\zm{y_i}\mtx{X}_i}-\zm{\yavg}\left(\frac{1}{n}\sum_{i=1}^n\X_i\right),\\
=&\frac{1}{n}\sum_{i=1}^n\zm{\zm{y_i}\mtx{X}_i}-\left(\frac{1}{n}\sum_{i=1}^n\zm{y_i}\right)\left(\frac{1}{n}\sum_{i=1}^n\X_i\right).
\end{align*}
Now using the triangular inequality for tensor spectral norm we conclude that
\begin{align*}
\tsn{\mtx{T}_n-\mtx{T}}\le \tsn{\frac{1}{n}\sum_{i=1}^n\zm{\zm{y_i}\mtx{X}_i}}+\tsn{\left(\frac{1}{n}\sum_{i=1}^n\zm{y_i}\right)\left(\frac{1}{n}\sum_{i=1}^n\X_i\right)}.
\end{align*}
We now state two lemmas to bound each of these terms. The proofs of these lemmas are defered to Sections \ref{thmdevp} and \ref{lemavgp}.
\begin{lemma}\label{thm dev} For $i=1,2,\ldots,n$ let $\vct{x}_i\in\R^p$ be i.i.d.~random Gaussian vectors distributed as $\mathcal{N}(\vct{0},\mtx{I}_p)$. Also let $\X_i\in\R^{\btd}$ be the tensorized version of $\vct{x}_i$ i.e.~$\vct{X}_i=\mathcal{T}(\vct{x}_i)$. Finally, assume $f:\R^p\mapsto \R$ is an $L$ Lipschitz function. Furthermore, assume $n\ge \left(\sum_{\ell=1}^D d_\ell\right)\log D$ and $D\ge 2$. Then
\[
\Pro\Bigg\{\tsn{\frac{1}{n}\sum_{i=1}^n \zm{\zm{f(\vct{x}_i)}\X_i}}\leq  \frac{c_1L}{\sqrt{n}}\left(\sqrt{\left(\sum_{\ell=1}^Dd_\ell\right)\log D}+t\right)\Bigg\}\leq e^{-\min\left(t^2,t\sqrt{n}\right)},
%|\sup_{\Tb\in\Ro}\li \sum_{i=1}^n\zm{\Y_i},\Tb\ri|\leq \order{{L\sum_i\sqrt{nd_i\log \Dp}}}.
\]
holds with $c_1>0$ a fixed numerical constant.
\end{lemma}
\begin{lemma} \label{lem avg}Consider the setup of Lemma \ref{thm dev}. Then
\[
\Pro\Bigg\{\tsn{\left(\frac{1}{n}\sum_{i=1}^n\zm{f(\vct{x}_i)}\right)\left(\frac{1}{n}\sum_{i=1}^n\X_i\right)}\geq \frac{c_2t_1L}{n}\left(\sqrt{\left(\sum_{\ell=1}^Dd_\ell\right)\log D}+t_2\right)\Bigg\}\leq 2\left(e^{-t_1^2}+e^{-t_2^2}\right),
\]
holds with $c_2>0$ a fixed numerical constant.
\end{lemma}
Combining Lemma \ref{thm dev} with $f=\func{}$, $L=\prod_{\ell=1}^D \twonorm{\vct{k}^{(\ell)}}$, and $c_1=c/2$ together with Lemma \ref{lem avg} with $t_1=\sqrt{n}$, $t_2=t$, and $c_2=c/2$ concludes the proof of Theorem \ref{finite sample}. All that remains is to prove Lemmas \ref{thm dev} and \ref{lem avg} which are the subject of the next two sections.
\subsubsection{Proof of Lemma \ref{thm dev}}
\label{thmdevp}
It is more convenient to carryout the steps of the proof on $\sum_{i=1}^n \zm{\zm{f(\vct{x}_i)}\X_i}$ in liue of $\frac{1}{n}\sum_{i=1}^n \zm{\zm{f(\vct{x}_i)}\X_i}$. The lemma trivally follows by a scaling by a factor $1/n$. We first write the tensor spectral norm as a supremum
\begin{align}
\tsn{\sum_{i=1}^n \zm{\zm{f(\vct{x}_i)}\X_i}}=\sup_{\Tb\in \Ro}\text{ }\abs{\sum_{i=1}^n\li \zm{\zm{f(\vct{x}_i)}\X_i},\Tb\ri}.\label{spec write}
\end{align}
Let $\Y_i=\zm{f(\vct{x}_i)}\X_i$. Define the random process $g(\Tb)=\sum_{i=1}^n\li \zm{\Y_i},\Tb\ri$. We claim that $g(\Tb)$ has a mixture of subgaussian and subexponential increments (see Definition \ref{ornorm} for subgaussian and subexponential random variables). Pick two tensors $\Tb,\Hb\in\R^{\btd}$. Increments of $g$ satisfy the linear relation
\[
g(\Tb)-g(\Hb)=\sum_{i=1}^n\li \zm{\Y_i},\Tb-\Hb\ri.
\]% (the vector obtained by arranging the entries of the tensor into a vector has unit Euclidean norm)
By construction $\E[g(\Tb)-g(\Hb)]=0$. We next claim that $\Y_i$ is a sub-exponential vector. Consider a tensor $\Tb$ with unit length $\tf{\Tb}=1$ i.e.~the sum of squares of entries are equal to one. We have $\li\Y_i,\Tb\ri=\zm{f(\vct{x}_i)}\li\X_i,\Tb\ri$. $f(\X_i)$ is a Lipschitz function of a Gaussian random vector. Thus, by the concentration of Lipschitz functions of Gaussians we have
\begin{align}
\Pro(|\zm{f(\X_i)}|\geq t)\leq 2\exp(-\frac{t^2}{2L^2}).\label{sg tail}
\end{align}
This immediately implies that $\tsub{\zm{f(\X_i)}}\leq cL$ for a fixed numerical constant $c$. Also note that $\li\X_i,\Tb\ri\sim\Nn(0,1)$ hence $\tsub{\li\X_i,\Tb\ri}\leq c$. These two identities combined with Lemma \ref{exp zm mult} implies a bound on the sub-exponential norm
\[
\te{\zm{\li\Y_i,\Tb\ri}}\leq c L.
\]
%for some constant $c>0$. 
Next, we observe that $g(\Tb)-g(\Hb)$ is sum of $n$ i.i.d. sub-exponentials each obeying $\te{\li \zm{\Y_i},\Tb-\Hb\ri}\leq cL\tf{\Tb-\Hb}$. Applying a standard sub-exponential Bernstein inequality, we conclude that
\begin{align}
\label{temptb}
\Pro(|g(\Tb)-g(\Hb)|\geq t)\leq 2\text{exp}\left(-c\cdot\min\left(\frac{t}{L\tf{\Tb-\Hb}},\frac{t^2}{nL^2\tf{\Tb-\Hb}^2}\right)\right),
\end{align}
holds with $\gamma$ a fixed numerical constant. This tail bound implies that $g$ is a mixed tail process that is studied by Talagrand and others \cite{talagrand2014gaussian,dirksen2013tail}. In particular, supremum of such processes are characterized in terms of a linear combination of Talagrand's $\gamma_1$ and $\gamma_2$ functionals (see Definition \ref{gamma functional} as well as \cite{talagrand2014gaussian,talagrand2006generic} for an exposition).
%$\gamma_\alpha$ functional of a set $S$ with respect to distance metric $d$ is given by,
%\[
%\gamma_\alpha(S,d)=
%\]
We pick the following distance metrics on tensors induced by the Frobenius norm: $d_1(\Tb,\Hb)=L\tf{\Hb-\Tb}/c$ and $d_2(\Tb,\Hb)=\tf{\Hb-\Tb}L\sqrt{n/c}$. We can thus rewrite \eqref{temptb} in the form
\[
\Pro\big\{|g(\Tb)-g(\Hb)|\geq t\big\}\leq 2\exp\left(-\min\left(\frac{t}{d_1(\Tb,\Hb)},\frac{t^2}{d_2^2(\Tb,\Hb)}\right)\right),
\]
which implies $\Pro\big\{|g(\Tb)-g(\Hb)|\geq \sqrt{t}d_2(\Tb,\Hb)+td_1(\Tb,\Hb)\big\}\leq 2\exp(-t)$. Observe that the radius of $\Ro$ with respect to $\tf{\cdot}$ norm is $1$ hence radius with respect to $d_1,d_2$ metrics are $L/c,L\sqrt{n/c}$ respectively. Applying Theorem $3.5$ of Dirksen \cite{dirksen2013tail}, we obtain
\[
\Pro\Bigg\{\sup_{\Tb\in\Ro}|g(\Tb)|\geq C\left(\gamma_2(\Ro, d_2)+\gamma_1(\Ro, d_1)+L\sqrt{un/c}+uL/c\right)\Bigg\}\leq e^{-u}.
\]
Observe that we can use the change of variable $t=L\cdot\max\left(\sqrt{un},u\right)$ to obtain
\begin{align}
\Pro\Bigg\{\sup_{\Tb\in\Ro}|g(\Tb)|\geq C\left(\gamma_2(\Ro, d_2)+\gamma_1(\Ro, d_1)+t\right)\Bigg\}\leq \exp\left(-\min\left(\frac{t^2}{L^2n},\frac{t}{L}\right)\right),\label{tail bound}
\end{align}
with some updated constant $C>0$. To conclude, we need to bound the $\gamma_2$ and $\gamma_1$ terms. To achieve this we will upper bound the $\gamma_\alpha$ functional in terms of Dudley's entropy integral which is stated in Lemma \ref{dudley lem}. First, let us find the $\eps$ covering number of $\Ro$. Pick $0<\delta\leq 1$ coverings $\Cc_\ell$ of the unit $\el_2$ balls $\Bc^{d_\ell}$. These covers have size at most $(1+2/\delta)^{d_\ell}$. Consider the set of rank $1$ tensors $\Cc=\Cc_1\bt\dots\bt\Cc_D$ with size $(1+2/\delta)^{\sum_{\ell=1}^Dd_\ell}$. For any $\robt{\vb}\in\Ro$, we can pick $\robt{\ub}\in\Cc$ satisfying $\tn{\vb_\ell-\ub_\ell}\leq\delta$ for all $1\leq \ell\leq D$. This implies
\begin{align}
\tf{\ten-\uten}&\leq \sum_{\ell=1}^D \tf{(\vb_1,\dots, \vb_\ell,\ub_{\ell+1},\dots,\ub_D)-(\vb_1,\dots, \vb_{\ell-1},\ub_{\ell},\dots,\ub_D)}\\
&=\sum_{\ell=1}^D \tn{\vb_\ell}\dots\tn{\vb_{\ell-1}}\tn{\vb_\ell-\ub_\ell}\tn{\ub_{\ell+1}}\dots\tn{\ub_D}\leq D\delta.
\end{align}
Denoting Frobenius norm covering number of $\Ro$ by $N(\eps)$, this implies that, for $0<\eps\leq 1$,
\[
N(\eps)\leq (1+2D/\eps)^{\sum_{\ell=1}^Dd_\ell}.
\]
Clearly, $N(\eps)=1$ for $\eps\geq 1$ by picking the cover $\{0\}$. Consequently,
\begin{align}
\gamma_1(\Ro,\tf{\cdot})\leq \int_{0}^1 \log N(\eps)d\eps\leq c\left(\sum_{\ell=1}^Dd_\ell\right)\log D,~\gamma_2(\Ro,\tf{\cdot})\leq \int_{0}^1 \sqrt{\log N(\eps)}d\eps\leq c\sqrt{\left(\sum_{\ell=1}^Dd_\ell\right)\log D}.\label{tf gamma}
\end{align}
Thus the metrics $d_1,d_2$ metrics are $\tf{\cdot}$ norm scaled by a constant. Hence, their $\gamma_\alpha$ functions are scaled versions of \eqref{tf gamma} given by
\[
\gamma_1(\Ro,d_1)\leq cL\left(\sum_{\ell=1}^Dd_\ell\right)\log D,~\gamma_2(\Ro,d_2)\leq cL\sqrt{n\left(\sum_{\ell=1}^Dd_\ell\right)\log D}.
\]
Now, observe that if $n\geq \left(\sum_{\ell=1}^Dd_\ell\right)\log D$, we have $\gamma_1(\Ro,d_1)\leq cL\sqrt{n\left(\sum_{\ell=1}^Dd_\ell\right)\log D}$. Substituting these in \eqref{tail bound} and using $n\geq \left(\sum_{\ell=1}^Dd_\ell\right)\log D$ we find
\[
\Pro\Bigg\{\sup_{\Tb\in\Ro}\abs{g(\Tb)}\geq C\left(L\sqrt{n\left(\sum_{\ell=1}^Dd_\ell\right)\log D}+t\right)\Bigg\}\leq e^{-\min\left(\frac{t^2}{L^2n},\frac{t}{L}\right)}.
\]
Substituting $t\rightarrow L\sqrt{n}t$ and recalling \eqref{spec write}, concludes the proof.

\subsubsection{Proof of Lemma \ref{lem avg}}
\label{lemavgp}
We first rewrite,
\begin{align}
\tsn{\left(\frac{1}{n}\sum_{i=1}^n\zm{f(\vct{x}_i)}\right)\left(\frac{1}{n}\sum_{i=1}^n\X_i\right)}= \abs{\frac{1}{n}\sum_{i=1}^n\zm{f(\vct{x}_i)}}\cdot\tsn{\frac{1}{n}\sum_{i=1}^n\X_i}.\label{rewrite}
\end{align}
As discussed in \eqref{sg tail}, $\tsub{\zm{f(\vct{x}_i)}}\leq cL$ for $c$ a fixed numerical constant. Since $\vct{x}_i$'s are i.i.d the empirical sum $\fb_{\text{avg}}=\frac{1}{n}\sum_{i=1}^n\zm{f(\vct{x}_i)}$ obeys the bound $\tsub{\fb_{\text{avg}}}\leq cL/\sqrt{n}$ as well. Hence, 
\begin{align}
\label{tt1}
\Pro\Big\{\abs{\fb_{\text{avg}}}\geq c'\frac{t_1L}{\sqrt{n}}\Big\}\leq 2e^{-t_1^2}. 
\end{align}
Also note that $\frac{1}{\sqrt{n}}\sum_{i=1}^n\X_i$ is a tensor with standard normal entries, applying \cite[Theorem 1]{tomioka2014spectral} we conclude that
\begin{align}
\label{tt2}
\tsn{\frac{1}{\sqrt{n}}\sum_{i=1}^n\X_i}\leq c''\left(\sqrt{\left(\sum_{\ell=1}^Dd_\ell\right) \log D}+t_2\right)\quad\Rightarrow\quad \tsn{\frac{1}{n}\sum_{i=1}^n\X_i}\leq c''\frac{\sqrt{\left(\sum_{\ell=1}^Dd_\ell\right) \log D}+t_2}{\sqrt{n}},
\end{align}
holds with probability $1-2e^{-t_2^2}$. Combining \eqref{tt1} and \eqref{tt2} via the union bound together with \eqref{rewrite} concludes the proof.

\subsection{Rank one approximation of the population tensor (Proof of Theorem \ref{thm pop})}
We begin the proof of this theorem by a few definitions regarding non-overlapping CNN models that simplify our exposition. For these definitions it is convenient to view non-overlapping CNNs as a tree with the root of the tree corresponding to the output of the CNN and the leaves corresponding to input features. In this visualization $D-\ell$th layer of the tree corresponds to the $\ell$th layer. Figure \ref{NCNNtree} depicts such a tree visualization along with the definitions discussed below.
\begin{figure}
\centering
\begin{tikzpicture}
\node at (-2.5,0) {\includegraphics[scale=0.5]{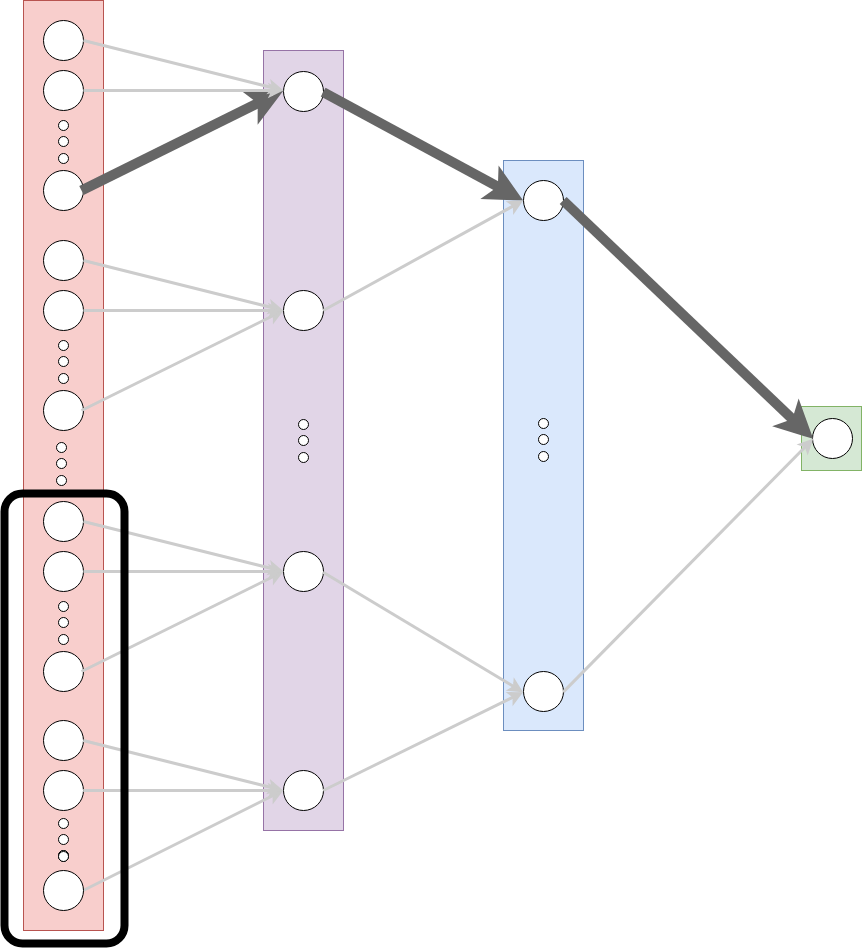}};
%\node[black] at (0.6,1.25) {$\vct{k}^{(1)}\in\R^{d_1}$};
%\node[black] at (-4.75,0.85) {$\vct{k}^{(2)}\in\R^{d_2}$};
%\node[black] at (-8.95,0.6) {$\vct{k}^{(3)}\in\R^{d_3}$};
%\node[black] at (2,-1.1) {$\vct{K}^{(1)}\in\R^{p_1\times p}$};
%\node[black] at (-4.65,-0.8) {$\vct{K}^{(2)}\in\R^{p_2\times p_1}$};
%\node[black] at (-8.95,-0.65) {\scriptsize{$\vct{K}^{(3)}\in\R^{1\times p_2}$}};
%\node[black] at (-2.35,-0.09) {\huge{$\phi_1$}};
%\node[black] at (-7.05,-0.09) {\huge{$\phi_2$}};
%\node[black] at (5.4,0.25){\small{$\vct{x}\in\R^p$}};
%\node[black] at (5.4,-0.29) {\small{$\vct{h}^{(0)}$}}; 
%\node[black] at (-1.525,0.25) {\small{$\bar{\vct{h}}^{(1)}$}};
%\node[black] at (-3.1,0.25) {\small{$\vct{h}^{(1)}$}};
%%\node[black] at (-1.525,-0.35) {\small{$\vct{x}_{\vct{w}}$}}; 
%\node[black] at (-6.215,0.25) {\small{$\bar{\vct{h}}^{(2)}$}};
%\node[black] at (-7.76,0.25) {\small{$\vct{h}^{(2)}$}};
%\node[black] at (-10.35,0.25) {\small{$\func{\vct{x}}~$}};
%\node[black] at (-10.45,-0.35) {\small{$\vct{h}^{(D)}~~$}};
\node[black] at (-8.9,9) {$\vct{x}=\vct{h}^{(0)}\in\R^{p}:=\R^{p_0}$};
\node[black] at (-4.7,8) {$\vct{h}^{(1)}\in\R^{p_1}$};
\node[black] at (-6.7,8) {$\vct{k}^{(1)}\in\R^{d_1}$};
\node[black] at (-6.5,3.7) {$\vct{k}^{(1)}$};
\node[black] at (-6.5,-0.9) {$\vct{k}^{(1)}$};
\node[black] at (-6.5,-4.8) {$\vct{k}^{(1)}$};
\node[black] at (-0.4,6) {$\vct{h}^{(2)}\in\R^{p_2}$};
\node[black] at (-2.5,6.5) {$\vct{k}^{(2)}\in\R^{d_2}$};
\node[black] at (-2,-2.7) {$\vct{k}^{(2)}$};
\node[black] at (5.5,1.8) {$y=\func{\vct{x}}=\vct{h}^{(D)}=\vct{h}^{(3)}$};
\node[black] at (2.5,3.5) {$\vct{k}^{(3)}\in\R^{d_3}$};
\node[black] at (-6.8,-8) {\Large{$\text{set}_2(p_2)$}};
%\node[black] at (-10.35,-1.25) {\small{$=\bar{\vct{h}}^{(3)}$}};
\end{tikzpicture}
\caption{Tree visualization of a non-overlapping CNN model. The path in bold corresponds to the path vector $\ib=(d_1,1,1,1)$ from Definition \ref{pathvect}. For this path the kernel path gain is equal $k_{\ib}=\vct{k}_{d_1}^{(1)}\vct{k}_{1}^{(2)}\vct{k}_{1}^{(3)}$ and the activation path gain is equal to $\sigmap_\ib(\x)=\phi_1'(\bar{h}_{1}^{(1)})\phi_2'(\bar{h}_{1}^{(2)})\phi_3'(\bar{h}_{1}^{(3)})$. The set $\text{set}_2(p_2)$ (offsprings of the last hidden node in layer two) is outlined.}
\label{NCNNtree}
\end{figure}
\vspace{-0.2cm}
\begin{definition}[Path vector]\label{pathvect} A vector $\ib\in\R^{D+1}$ is a path vector if its zeroth coordinate satisfies $1\leq \ib_0\leq p$ and for all $D-1\geq j\geq 0$, $\ib_{j+1}$ obeys $\ib_{j+1}=\left\lceil{\ib_{j}}/{d_j}\right\rceil$. This implies $1\leq \ib_j\leq p_j$ and $\ib_D=1$. We note that in the tree visualization a path vector corresponds to a path connecting a leaf (input feature) to the root of the tree (output). We use $\Ic$ to denote the set of all path vectors and note that $|\Ic|=p$. We also define $\ib(i)\in\Ic$ be the vector whose zeroth entry is $\ib_0=i$. Stated differently $\ib(i)$ is the path connecting the input feature $i$ to the output. Given a path $\ib$ and a $p$ dimensional vector $\vb$, we define $\vb_{\ib}:=\vb_{\ib_0}$. A sample path vector is depicted in bold in Figure \ref{NCNNtree} which corresponds to $\ib=(d_1,1,1,1)$.
%Let $\hib$ be the vector obtained by $\hib_\el=\modu(\ib_\el,d_\el)$ for $1\leq \el\leq D$. 
\end{definition}
\vspace{-0.2cm}
\begin{definition}[Kernel and activation path gains]\label{g path} Consider a CNN model of the form \eqref{CNNform} with input $\vct{x}$ and inputs of hidden units given by $\{\xbr^{(\el)}\}_{\el=1}^D$. To any path vector $\vct{i}$ connecting an input feature to the output we associate two path gains: a kernel path gain denoted by $\lan_{\ib}$ and activation path gain denoted by $\sigmap_\ib(\x)$ defined as
\[
\lan_{\ib}=\prod_{\el=1}^{D} \lay{\el}_{\modu(\ib_{\el-1},d_\el)}\quad\text{and}\quad\sigmap_\ib(\x)=\prod_{\el=1}^{D}\sigmap_\el(\xbr^{(\el)}_{\ib_\el}),
\]
where $\modu(a,b)$ denotes the remainder of dividing integer $a$ by $b$. In words the kernel path gain is multiplication of the kernel weights along the path and the activation path gain is the multiplication of the derivatives of the activations evaluated at the hidden units along the path. For the path depicted in Figure \ref{NCNNtree} in bold the kernel path gain is equal $k_{\ib}=\vct{k}_{d_1}^{(1)}\vct{k}_{1}^{(2)}\vct{k}_{1}^{(3)}$ and the activation path gain is equal to $\sigmap_\ib(\x)=\phi_1'(\bar{h}_{1}^{(1)})\phi_2'(\bar{h}_{1}^{(2)})\phi_3'(\bar{h}_{1}^{(3)})$.
\end{definition}
\begin{definition}[CNN offsprings] Consider a CNN model of the form \eqref{CNNform} with input $\vct{x}$ and inputs of hidden units given by $\{\xbr^{(\el)}\}_{\el=1}^D$. We will associate a set $\text{set}_\el(i)\subset \{1,\dots,p\}$ to the $i$th hidden unit of layer $\ell$ defined as $\text{set}_\el(i)=\{(i-1)r_\el+1,(i-1)r_\el+2,\dots,ir_\el\}$ where $r_\el=p/p_\el$. By construction, this corresponds to the set of entries of the input data $\x$ that $\xbr^{(\el)}_i(\x)$ is dependent on. In the tree analogy $\text{set}_\el(i)$ are the leaves of the tree connected to hidden unit $i$ in the $\ell$th layer i.e.~the set of offsprings of this hidden node. We depict $\text{set}_2(p_2)$ which are the offsprings of the last hidden node in layer two in Figure \ref{NCNNtree}.
\end{definition}
%Denote $\funcp{\x}=\frac{\pa \func{\x}}{\pa \x}\in\R^p$ and its $\ib_0$th entry by $\funcp{\x,\ib}$.
We now will rewrite the population tensor in a form such that it is easier to see why it can be well approximated by a rank one tensor. Note that since the tensorization operation is linear the population tensor is equal to
\begin{align}
\label{teniden1}
\mtx{T}=\E[\func{\vct{x}}\mtx{X}]=\E[\func{\vct{x}}\mathcal{T}(\vct{x})]=\mathcal{T}\left(\E[\func{\vct{x}}\vct{x}]\right).
\end{align}
 Define the vector $\fcp$ to be the population gradient vector i.e.~$\fcp=\E[\nabla \func{\x}]$ and note that Stein's lemma combined with \eqref{teniden1} implies that
\begin{align}
\label{stgrad}
\mtx{T}=\mathcal{T}\left(\E[\func{\vct{x}}\vct{x}]\right)=\mathcal{T}\left(\E[\nabla\func{\vct{x}}]\right)=\mathcal{T}\left(\fcp\right).
\end{align}
Also note that 
\[
%\funcp{\x,\ib}
\frac{\pa\func{\x}}{\pa \x_{i}}=\lan_{\ib(i)}\sigmap_{\ib(i)}(\x).\label{ith entry}
\]
 Thus we have
\begin{align}
\fcp_i=\E\bigg[\frac{\pa\func{\x}}{\pa \x_{i}}\bigg]=\lan_{\ib(i)}\E[\sigmap_{\ib(i)}(\x)].\label{gradent}
\end{align}
Define a vector $\vct{k}\in\R^p$ such that $\vct{k}_i=\lan_{\ib(i)}$. Since $\lan_{\ib(i)}$ consists of the product of the kernel values across the path $\ib(i)$ it is easy to see that the tensor $\mtx{K}:=\mathcal{T}(\vct{k})$ is equal to
\begin{align}
\label{kouter}
\mtx{K}=\bt_{\el=1}^D\lay{\el}.
\end{align}
Similarly define the vector $\vct{v}\in\R^p$ such that $\vct{v}_i=E[\sigmap_{\ib(i)}(\x)]$ and define the corresponding tensor $\mtx{V}=\mathcal{T}(\vct{v})$. Therefore, \eqref{gradent} can be rewritten in the vector form $\fcp=\vct{k}\odot \vct{v}$ where $\vct{a}\odot\vct{b}$ denotes entry-wise (Hadamard) product between two vectors/matrices/tensors $\vct{a}$ and $\vct{b}$ of the same size. Thus using \eqref{stgrad} the population tensor can alternatively be written as
\begin{align*}
\mtx{T}=\mathcal{T}\left(\fcp\right)=\mtx{K}\odot \mtx{V}=\left(\bt_{\el=1}^D\lay{\el}\right)\odot\mtx{V}.
\end{align*}
Therefore, the population tensor $\mtx{T}$ is the outer product of the convolutional kernels whose entries are masked with another tensor $\mtx{V}$. If the entries of the tensor $\mtx{V}$ were all the same the population tensor would be exactly rank one with the factors revealing the convolutional kernel. One natural choice for approximating the population tensor with a rank one matrix is to replace the masking tensor $\mtx{V}$ with a scalar. That is, use the approximation $\mtx{T}\approx c \bt_{\el=1}^D\lay{\el}$. Recall that $\lcnn:=\cgain\bt_{\el=1}^D\lay{\el}$ is exactly such an approximation with $c$ set to $\cgain$ given by 
\[
\cgain=\prod_{\el=1}^{D}\E_{\x\sim\Nn(0,\Iden)}[\sigmap_\el(\xbr^{(\el)}_{\ib_\el})].
\]%\footnote{This claim is straightforward to show for non-overlapping CNN by applying a standard induction argument over layers.}
To characterize the quality of this approximation note that
\begin{align*}
\tsn{\mtx{T}-\lcnn}\overset{(a)}{\le}& \fronorm{\mtx{T}-\lcnn},\\
\overset{(b)}{=}&\fronorm{\mtx{K}\odot\left(\mtx{V}-\cgain\right)},\\
\overset{(c)}{=}&\twonorm{\vct{k}\odot\left(\vct{v}-\cgain\right)},\\
\overset{(d)}{\le}& \twonorm{\vct{k}}\infnorm{\vct{v}-\cgain},\\
\overset{(e)}{=}&\lipp\infnorm{\vct{v}-\cgain}.
\end{align*}
Here, $(a)$ follows from the fact for a tensor, its spectral norm is smaller than its Frobenius norm, (b) from the definitions of $\mtx{T}$ and $\lcnn$, (c) from the fact that $\mtx{K}=\mathcal{T}(\vct{k})$ and $\mtx{V}=\mathcal{T}(\vct{v})$, (d) from the fact that $\twonorm{\vct{a}\odot \vct{v}}\le \twonorm{\vct{a}}\infnorm{\vct{b}}$, and (e) from the fact that the Euclidean norm of the kronecker product of of vectors is equal to the product of the Euclidean norm of the indivisual vectors. As a result of the latter inequality to prove Theorem \ref{thm pop} it suffices to show that 
\begin{align}
\label{impinter}
\infnorm{\vct{v}-\cgain}\le \sqrt{8\pi}\mu\smo\cdot\underset{\ell}{\sup} \prod_{i=1}^\ell \twonorm{\vct{k}^{(i)}}\cdot \frac{D}{\sqrt{\underset{ \ell }{\min}\text{ }d_\ell}}.
\end{align}
In the next lemma we prove a stronger statement.
\begin{lemma}\label{infnormbnd} Assume the activations are $S$-smooth. Also consider a vector $\vct{v}\in\R^p$ with entries $\vct{v}_i=E_{\x\sim\Nn(0,\Iden)}[\sigmap_{\ib(i)}(\x)]$ and $\cgain=\prod_{\el=1}^{D}\E_{\x\sim\Nn(0,\Iden)}[\sigmap_\el(\xbr^{(\el)}_{\ib_\el})]$ we have
\begin{align*}
\abs{\vct{v}_i-\cgain}\le \kappa_{i}:=\sqrt{8\pi}S\sum_{\el=1}^D\abs{\lay{\el}_{\modu(\ib_{\el-1},d_\el)}}\prod_{i=1}^{\el-1}\tn{\lay{i}}.
\end{align*}
Here, $\vct{i}$ is the vector path that starts at input feature $i$.
\end{lemma}
Before proving this lemma let us explain how \eqref{impinter} follows from this lemma. To show this we use the kernel diffuseness assumption introduced in Definition \ref{assume diff}. This definition implies that $\abs{\lay{\el}_{\modu(\ib_{\el-1},d_\el)}}\le\infnorm{\vct{k}^{(\ell)}}\le \frac{\mu}{\sqrt{d_\ell}}\twonorm{\vct{k}^{(\ell)}}$. Thus we have
\begin{align*}
\kappa_i:=&\sqrt{8\pi}S\sum_{\el=1}^D|\lay{\el}_{\modu(\ib_{\el-1},d_\el)}|\prod_{i=1}^{\el-1}\tn{\lay{i}},\\
\le&\sqrt{8\pi}\mu S\sum_{\el=1}^D\frac{1}{\sqrt{d_\ell}}\twonorm{\vct{k}^{(\ell)}}\prod_{i=1}^{\el-1}\tn{\lay{i}},\\
=&\sqrt{8\pi}\mu S\sum_{\el=1}^D\frac{1}{\sqrt{d_\ell}}\prod_{i=1}^{\el}\tn{\lay{i}},\\
\le&\sqrt{8\pi}\mu D S\cdot \underset{\ell}{\sup} \frac{\prod_{i=1}^{\el}\tn{\lay{i}}}{\sqrt{d_\ell}},\\
\le&\sqrt{8\pi}\mu\smo\cdot\underset{\ell}{\sup} \prod_{i=1}^\ell \twonorm{\vct{k}^{(i)}}\cdot \frac{D}{\sqrt{\underset{ \ell }{\min}\text{ }d_\ell}}.
\end{align*}
This completes the proof of Theorem \ref{thm pop}. All that remains is to prove Lemma \ref{infnormbnd} which is the subject of the next section.
\subsubsection{Proof of Lemma \ref{infnormbnd}}
To bound the difference between $\vct{v}_i$ and $\cgain$ consider the path $\ib=\ib(i)$ and define the variables $\{a_i\}_{i=0}^D$ as
\[
a_i=\E_{\x\sim\Nn(0,\Iden)}[\prod_{\el=1}^{i}\sigmap_\el(\xbr^{(\el)}_{\ib_\el})]\prod_{\el=i+1}^{D}\E_{\x\sim\Nn(0,\Iden)}[\sigmap_\el(\xbr^{(\el)}_{\ib_\el})].
\]
Note that $a_D=\vct{v}_i$ and $a_0=\cgain$. To bound the difference $\vct{v}_i-\cgain=a_D-a_0$ we use a telescopic sum
\begin{align}
|a_D-a_0|\leq \sum_{\el=0}^{D-1} |a_{\el+1}-a_\el|.\label{all upper}
\end{align}
We thus focus on bounding each of the summands $|a_\el-a_{\el-1}|$. Setting $\gamma_\el=\prod_{i=\el}^{D}\E_{\x\sim\Nn(0,\Iden)}[\sigmap_i(\xbr^{(i)}_{\ib_i})]$, this can be written as
\[
a_\el-a_{\el-1}=\E_{\x\sim\Nn(0,\Iden)}[(\sigmap_\el(\xbr^{(\el)}_{\ib_\el})-\E[\sigmap_\el(\xbr^{(\el)}_{\ib_\el})])\prod_{i=1}^{\el-1}\sigmap_\el(\xbr^{(\el)}_{\ib_\el})]\gamma_{\el+1}.
\]
Using $|\gamma_\el|\leq 1$ (which follows from the assumption that the activations are $1$-Lipschitz), it suffices to bound 
\begin{align}
\label{funkyinter}
\E_{\x\sim\Nn(0,\Iden)}[(\sigmap_\el(\xbr^{(\el)}_{\ib_\el})-\E[\sigmap_\el(\xbr^{(\el)}_{\ib_\el})])\theta_{\el-1}]\quad\text{where}\quad\theta_{\el}=\prod_{i=1}^{\el}\sigmap_i(\xbr^{(i)}_{\ib_i}). 
\end{align}
To  this aim we state two useful lemmas whose proofs are deferred to Sections \ref{xyzpf} and \ref{det funcpf}.
\begin{lemma} \label{xyz}Let $X,Y,Z$ be random variables where $X$ is independent of $Z$. Let $f$ be an $L$-Lipschitz function. Then
\begin{align}
|\E[f(X+Y)Z]-\E[f(X+Y)]\E[Z]|\leq  L\E[|\zm{Y}|(|Z|+|\E [Z]|)].\label{der bound}
\end{align}
Furthermore if $|Z|\leq1$, then 
\begin{align}
|\E[f(X+Y)Z]-\E[f(X+Y)]\E[Z]|\leq 2L\E[|\zm{Y}|].\label{der bound2}
\end{align}
\end{lemma}
\begin{lemma} \label{det func}$\xbr^{(\el)}_i(\x)$ (and $\hh^{(\el)}_i(\x)$) is a deterministic function of the entries of $\x$ indexed by $\text{set}_\el(i)$. In other words, there exists a function $f$ such that $\xbr^{(\el)}_i(\x)=f(\x_{\text{set}_\el(i)})$.
\end{lemma}
With these lemmas in-hand we return to bounding \eqref{funkyinter}. To this aim we decompose $\xbr^{(\el)}_{\ib_\el}$ as follows
\[
\xbr^{(\el)}_{\ib_\el}=\sum_{i=1}^{d_{\el}}\lay{\el}_i\hh^{(\el-1)}_{d_{\el}(\ib_\el-1)+i}=\lay{\el}_{\modu(\ib_{\el-1},d_\el)}\hh^{(\el-1)}_{\ib_{\el-1}}+\rest,
\]
where the $\rest$ term is the contribution of the entries of $\hh^{(\el-1)}$ other than $\ib_{\el-1}$. By the non-overlapping assumption, $\rest$ is independent of $\theta_{\el-1}$ as well as $\hh^{(\el-1)}_{\ib_{\el-1}}$ (see Lemma \ref{det func}). In particular, $\hh^{(\el-1)}_{\ib_{\el-1}}$ and $\theta_{\el-1}$ is a function of $\x_{\text{set}_{\el-1}(\ib_{\el-1})}$ whereas $\rest$ is a function of the entries over the complement $\text{set}_{\el}(\ib_\el)-\text{set}_{\el-1}(\ib_{\el-1})$. With these observations, applying Lemma \ref{xyz} with $f=\sigmap_\el,X=\rest, Y=\lay{\el}_{\modu(\ib_{\el-1},d_\el)}\hh^{(\el-1)}_{\ib_{\el-1}},Z=\theta_{\el-1}$ and using the fact that $|\theta_{\el-1}|\leq 1$ which holds due to $1$-Lipschitzness of $\sigma_\el$'s, we conclude that
\[
|\E_{\x\sim\Nn(0,\Iden)}[(\sigmap_\el(\xbr^{(\el)}_{\ib_\el})-\E[\sigmap_\el(\xbr^{(\el)}_{\ib_\el})])\theta_{\el-1}]|\leq 2\smo\E|\zm{Y}|.
\]
Here, $\smo$ is the smoothness of $\sigma_\el$ and Lipschitz constant of $\sigmap_\el$. To conclude, we need to assess the $\E|\zm{Y}|$ term. Now note that starting from $\x$, each entry of $\hh^{(\el-1)}$ is obtained by applying a sequence of inner products with $\{\lay{i}\}_{i=1}^{\el-1}$ and activations $\sigma_\el(\cdot)$, which implies $\hh^{(\el-1)}_{\ib_{\el-1}}$ is a $\prod_{i=1}^{\el-1}\tn{\lay{i}}$-Lipschitz function of $\text{set}_{\el-1}(\ib_{\el-1})$. This implies $Y$ is a Lipschitz function of a Gaussian vector with Lipschitz constant $L_Y=|\lay{\el}_{\modu(\ib_{\el-1},d_\el)}|\prod_{i=1}^{\el-1}\tn{\lay{i}}$. Hence, $\zm{Y}$ obeys the tail bound
\[
\Pro(|\zm{Y}|\geq t)\leq 2\exp(-\frac{t^2}{2L_Y^2}).
\]
Using a standard integration by parts argument the latter implies that
\[
\E|\zm{Y}|\leq \sqrt{2\pi} L_Y.
\]
Thus, 
\begin{align*}
\abs{a_\el-a_{\el-1}}\le|\lay{\el}_{\modu(\ib_{\el-1},d_\el)}|\prod_{i=1}^{\el-1}\tn{\lay{i}}.
\end{align*}
concluding the upper-bound on each summand of \eqref{all upper}. Combining such upper bounds \eqref{all upper} implies
\[
|a_D-a_0|\leq \sqrt{8\pi}S\sum_{\el=1}^D|\lay{\el}_{\modu(\ib_{\el-1},d_\el)}|\prod_{i=1}^{\el-1}\tn{\lay{i}}:=\kappa_{\ib}.
\]
This concludes the proof of Lemma \ref{infnormbnd}.
\paragraph{Proof of Lemma \ref{xyz}}
\label{xyzpf}
Using the independence of $X,Z$, we can write 
\begin{align*}%-\E[\sigmap(X+Y)]Z
\E[f(X+Y)Z]&=\E[f(X+\E[Y])Z]+\E[(f(X+Y)-f(X+\E[Y]))Z]\\
&=\E[f(X+\E[Y])]\E[Z]+\E[(f(X+Y)-f(X+\E[Y]))Z]\\
&=\E[f(X+Y)]\E[Z]+\E[f(X+\E[Y])-f(X+Y)]\E[Z]+\E[(f(X+Y)-f(X+\E[Y]))Z].
\end{align*}
This implies 
\[
\E[f(X+Y)Z]-\E[f(X+Y)]\E[Z]=\E[f(X+\E[Y])-f(X+Y)]\E[Z]+\E[(f(X+Y)-f(X+\E[Y]))Z].
\]
Now, using Lipschitzness of $f$, we deterministically have that $|f(X+\E[Y])-f(X+Y)|\leq L|Y-\E[Y]|=L|\zm{Y}|$. Similarly, $|(f(X+Y)-f(X+\E[Y]))Z|\leq L|\zm{Y}Z|$. Taking absolute values we arrive at
\[
|\E[f(X+\E[Y])-f(X+Y)]\E[Z]+\E[(f(X+Y)-f(X+\E[Y]))Z]|\leq  L\E[|\zm{Y}|(|Z|+|\E [Z]|)].
\]
This immediately implies \eqref{der bound}. If $|Z|\leq 1$ almost surely, we have $\E[|\zm{Y}Z|]\leq\E[|\zm{Y}|]$ and $|\E [Z]|\leq 1$ which yields the $2L\E[|\zm{Y}|]$ upper bound.
\paragraph{Proof of Lemma \ref{det func}}
\label{det funcpf}
Informally, this lemma is obvious via the tree visualization. To formally prove this lemma we use  an induction argument. For $\xbr^{(1)}$ the result is trivial because $\xbr^{(1)}_i=\li\lay{1},\x(i)\ri$ which is a weighted sum of entries corresponding to $\text{set}_1(i)$. Suppose the claim holds for all layers less than or equal to $\el-1$ and $\xbr^{(\el-1)}_i=f_{\el-1}(\x_{\text{set}_{\el-1}(i)})$. For layer $\el$, we can use the fact that $\text{set}_{\el}(i)=\bigcup_{j=1}^{d_\el}\text{set}_{\el-1}((i-1)d_\el+j)$ to conclude that
\begin{align*}
\xbr^{(\el)}_i&=\sum_{j=1}^{d_\el} \lay{\el}_j\sigma_{\el-1}(\xbr^{(\el-1)}_{(i-1)d_\el+j})\\
&=\sum_{j=1}^{d_\el} \lay{\el}_j \sigma_{\el-1}(f_{\el-1}(\x_{\text{set}_{\el-1}((i-1)d_\el+j)})):=f_{\el}(\x_{\text{set}_{\el}(i)}).
\end{align*}
The latter is clearly a deterministic function of $\x_{\text{set}_{\el}(i)}$. Also it is independent of entry $i$ because it simply chunks the vector $\x_{\text{set}_{\el}(i)}$ into $d_\el$ sub-vectors and returns a sum of weighted functions of these sub-vectors. Here, the weights are the entries of $\lay{\el}$ and the functions are given by $\sigma_{\el-1}(f_{\el-1}(\cdot))$ (also note that the activation output is simply $\hh^{(\el)}_i=\sigma_\el(f_{\el}(\x_{\text{set}_{\el}(i)}))$).
\subsection{Proofs for learning the convolutional kernels (Proof of Theorem \ref{main thm})}
The first part of the theorem follows trivially by combining Theorems \ref{finite sample} and \ref{thm pop}. To translate a bound on the tensor spectral norm of $\mtx{T}_n-\lcnn$ to a bound on learning the kernels, requires a perturbation argument for tensor decompositions. This is the subject of the next lemma.
\begin{lemma} \label{nntd robust}Let $\mtx{L}=\gamma \robt{\vb}$ be a rank one tensor with $\{\vb_i\}_{i=1}^D$ vectors of unit norm. Also assume $\Eb$ is a perturbation tensor obeying $\tsn{\Eb}\leq \delta$. Set
\begin{align}
\label{deeptdalgtemp}
\vct{u}_1,\vct{u}_2,\ldots,\vct{u}_D=\underset{\bar{\vct{u}}_1,\bar{\vct{u}}_2,\ldots,\bar{\vct{u}}_D}{\arg\max} \li\mtx{L},\robt{\bar{\vct{u}}}\ri\quad\text{subject to}\quad \twonorm{\bar{\vct{u}}_1}=\twonorm{\bar{\vct{u}}_2}=\ldots=\twonorm{\bar{\vct{u}}_D}=1.
\end{align}
Then we have
\[
\prod_{i=1}^D|\ub_i^*\vb_i|\geq 1-2\delta/\gamma.
\]
\end{lemma}
The proof of Theorem \ref{main thm} is complete by applying Lemma \ref{nntd robust} above with $\vct{v}_\ell=\vct{k}^{(\ell)}$, $\vct{u}_\ell=\hat{\vct{k}}^{(\ell)}$, $\gamma=\cgain$ and $\mtx{E}=\mtx{T}_n-\lcnn$. All that remains is to prove Lemma \ref{nntd robust} which is the subject of the next section.

\subsubsection{Proof of Lemma \ref{nntd robust}}
To prove this lemma first note that for any two rank one tensors we have 
\[
\li\robt{\ub},\robt{\vb}\ri=\prod_{i=1}^D \li\ub_i,\vb_i\ri.
\]
Using this equality together with the fact that the vectors $\{\vct{u}_\ell\}_{\ell=1}^D$ are a maximizer for \eqref{deeptdalgtemp} we conclude that
\begin{align}
\label{innnnlem1}
\li\mtx{L}+\Eb,\robt{\vb}\ri&\leq \li\mtx{L}+\Eb,\robt{\ub}\ri,\nonumber\\
&\leq \abs{\li\mtx{L},\robt{\ub}\ri}+\abs{\li\Eb,\robt{\ub}\ri},\nonumber\\
&\leq \gamma\prod_{i=1}^D |\li\ub_i,\vb_i\ri|+\delta.
\end{align}
Furthermore, note that
\begin{align}
\label{innnnlem2}
\li\mtx{L}+\Eb,\robt{\vb}\ri=\gamma+\li\Eb,\robt{\vb}\ri\geq \gamma-\delta.
\end{align}
Combining \eqref{innnnlem1} and \eqref{innnnlem2} we conclude that
\[
\gamma\prod_{i=1}^D |\li\ub_i,\vb_i\ri|+\delta\geq \gamma-\delta,
\]
concluding the proof.
\subsection{Generalization bounds for CERM}\label{sec cerm}
In this section we prove a generalization bound for Centered Empirical Risk Minimization (CERM) \eqref{cerm}. The following theorem shows that using a finite sample size $n$, CERM is guaranteed to choose a function close to population's minimizer. For the sake of this section $\|f\|_{L_\infty}$ will be the Lipschitz constant of a function.
\begin{theorem} \label{cerm thm}Let $\Fc$ be a class of functions $f:\R^p\rightarrow \R$. Suppose $\sup_{f\in \Fc}\|f\|_{L_\infty}\leq R$. Let $\{(\x_i,y_i)\}_{i=1}^n\sim(\x,y)$ be i.i.d. data points where $\x\sim\Nn(0,\Iden_p)$ and $y$ is so that $y-\E[y]$ is $K$ subgaussian. Suppose $\Fc$ has $\|\cdot\|_{L_\infty}$, $\delta$-covering bound obeying $\log N_\delta\leq s\log \frac{C}{\delta}$ for some constants $s\geq 1,C\geq R$. Given $\eps\leq\KB= K+R$, suppose $n\geq \order{\max\{\eps^{-1},\eps^{-2}\}\KB^2s\log \frac{C'p\KB}{\eps}} $ for some $C'>0$. Then the CERM output \eqref{cerm} obeys
\begin{align}
\E[\left(\hat{f}(\x)-y-\E[(\hat{f}(\x)-y)]\right)^2]\leq \min_{f\in \Fc}\E[\left(f(\x)-y-\E[(f(\x)-y)]\right)^2]+\eps.\label{cerm eq}
\end{align}
with probability $1-\exp(-\order{n})-4n\exp(-p)$.
\end{theorem}
\begin{proof} Consider the centered empirical loss that can alternatively be written in the form 
\begin{align}
E(f)&=\frac{1}{n}\sum_{i=1}^n \zm{y_i-f(\x_i)}^2-\frac{1}{n^2}(\sum_{i=1}^n\zm{ y_i-f(\x_i)})^2-\E[\zm{{f}(\x)-y}^2].\label{erm loss}\nonumber\\
&=\frac{1}{n}\sum_{i=1}^n \zm{\zm{y_i-f(\x_i)}^2}-\frac{1}{n^2}(\sum_{i=1}^n\zm{ y_i-f(\x_i)})^2\nonumber\\
&:=T_1+T_2
\end{align} 
To prove the theorem, we will simply bound the supremum $\sup_{f\in\Fc}|E(f)|\leq \sup_{f\in \Fc}|T_1+T_2|$. Pick a $\delta$ covering $\Fc_\delta$ of $\Fc$ with size $s\log \frac{C}{\delta}$ where $\delta$ will be determined later in this proof. We first bound $E(f)$ for all $f\in \Fc_\delta$. Given a fixed $f$, observe that $\te{\zm{\zm{y_i-f(\x_i)}^2}}\leq \order{\KB}^2=\order{K+R}^2$ which follows from $\tsub{\zm{y_i-f(\x_i)}}\leq \order{\KB}=\order{K+R}$. Applying subexponential concentration, since $T_1$ is sum of i.i.d.~subexponentials, we have
\begin{align}
\Pro(|T_1|\geq  \eps)\leq \exp(-\order{n\min\{\eps^2/\KB^2,\eps/\KB\}}).\label{t1 bound}
\end{align}
Next, since $\tsub{\zm{ y_i-f(\x_i)}}\leq \order{\KB}$, we can conclude that $\tsub{\frac{1}{n}\sum_{i=1}^n\zm{ y_i-f(\x_i)}}\leq \order{\KB}/\sqrt{n}\implies \te{T_2}\leq \order{\KB}^2/n$. Using \eqref{t1 bound} for $T_1$ and the subexponential tail bound for $T_2$ holds when $\eps\leq \KB$, and assuming the number of samples $n$ obeys $n\geq \order{\max\{\eps^{-1},\eps^{-2}\}\KB^2s\log \frac{C}{\delta}}$, then for all cover elements
\[
|T_1|+|T_2|\leq 2\eps.
\]
holds with probability at least $1-\exp(-\order{n})$. To conclude the proof, we need to move from the cover $\Fc_\delta$ to $\Fc$. Pick $f\in\Fc$ and its $\delta$ neighbor $f_\delta\in\Fc_\delta$. Utilizing the deterministic relation $|\zm{X}|\leq |X|+|\E[X]|$ and using the fact that $f_\delta$ is in a $\delta$ neighborhood of $f$, we arrive at the following bounds
%We have that 
%\[
%|\zm{f(\x)-y}^2-\zm{f_\delta(\x)-y}^2|\leq |\zm{f(\x)-f_\delta(\x)}\zm{f(\x)+f_\delta(\x)-2y}|\leq 4\eps_0 R\tn{\x}^2.
%\]
\begin{align}
&|\zm{f(\x)-f_\delta(\x)}|\leq  \delta (\tn{\x}+\E[\tn{\x}]).\label{rel delta}\\
&|\zm{f(\x)+f_\delta(\x)-2y}|\leq 2R(\tn{\x}+\E[\tn{\x}])+2K|\zm{y}|.
\end{align}
Next observe that, with probability at least $1-4n\exp(-p)$, all $\x_i,y_i$ obey $\tn{\x_i}\leq \order{\sqrt{p}},|\zm{y_i}|\leq \order{K\sqrt{p}}$. Combining this with \eqref{rel delta}, we conclude that for all $1\leq i\leq n$
\begin{align}
|\zm{f(\x_i)-y_i}^2-\zm{f_\delta(\x_i)-y_i}^2|\leq \order{\KB \delta p}.\label{eqeq1}
\end{align}
Expanding the square differences in the same way, an identical argument shows the following two deviation bounds,% on the $T_2$ term
\begin{align}
&|\E[\zm{f(\x_i)-y_i}^2-\zm{f_\delta(\x_i)-y_i}^2]|\leq \order{\KB \delta p},\label{eqeq2}\\
&\frac{1}{n^2}|(\sum_{i=1}^n\zm{ y_i-f(\x_i)})^2-(\sum_{i=1}^n\zm{ y_i-f_\delta(\x_i)})^2|\leq \order{\KB \delta p}\nn.
\end{align}
Combining these three inequalities (\eqref{eqeq1} and \eqref{eqeq2}) and substituting them in \eqref{erm loss}, we conclude that for all neighbors $f_\delta,f$,
\[
|E(f)-E(f_\delta)|\leq \order{\KB \delta p}.
\]
Next we set $\delta=c{\eps/(p\KB)}$ for a sufficiently small constant $c>0$, to find that with probability at least $1-\exp(-{n})$, $\sup_{f\in\Fc} |E(f)|\leq \eps$ holds as long as the number of samples $n$ obeys $n\geq \order{\max\{\eps^{-1},\eps^{-2}\}\KB^2s\log \frac{Cp\KB}{c\eps}}$. We define $\Lc_{erm}(f)=\frac{1}{n}\sum_{i=1}^n (y_i-f(\x_i)-\ravg{f})^2$ and $\Lc_{pop}(f)=\E[\zm{f(\x)-y}^2]$. We also denote the CERM minimizer $f_{erm}=\arg\min_{f\in\Fc} \Lc(f)$ and population minimizer $f_{pop}=\min_{f\in \Fc}\Lc_{pop}(f)$. Inequality \eqref{cerm eq} follows from the facts that we simultaneously have $|E(f_{erm})|\leq \order{\eps}$ and $|E(f_{pop})|\leq \order{\eps}$ which implies that
\begin{align*}\Lc_{pop}(f_{erm})\leq \Lc_{erm}(f_{erm})+\order{\eps}\leq \Lc_{erm}(f_{pop})+\order{\eps}\leq\Lc_{pop}(f_{pop})+\order{\eps},
\end{align*}
concluding the proof.
%using the fact that both $f_{erm}$ and the population minimizer $f_{pop}$ obeys .
\end{proof}

\subsection{Proof of Theorem \ref{ssa}}
In this section we will show how Theorem \ref{ssa} follows from Theorem \ref{cerm thm}. To this aim we need to show that $\Fc_{\hat{\lan},B}$ has a small Lipscshitz covering number. We construct the following cover $\Fc'$ for the set $\Fc_{\hat{\lan},B}$. Let $B'=B^{1/D}$. Pick a $\delta\leq B'$ $\ell_2$ cover $\Cc$ of the interval $[-B',B']$ which has size $2B'/\delta$. Let $\Cc_i$ be identical copies of $\Cc$. We set
\[
\Fc'=\{\func{\beta_\el \lah{\el}}\bgl \beta_\el\in \Cc_\el,~1\leq\el\leq D\}.
\]
In words, we construct CNNs by picking numbers from the cartesian product $\Cc_1\times \dots \Cc_D$ and scaling $\{\lah{\el}\}_{\el=1}^D$ with them. We now argue that $\Fc'$ provides a cover of $\Fc$. Given $f\in\Fc$ with scalings $\beta_\el$, there exists $f'\in\Fc'$ which uses scalings $\beta_\el'$ such that $|\beta_\el-\beta_\el'|\leq\delta$. Now, let $f_\el$ be the function with scalings $\beta'_i$ until $i=\el$ and $\beta_i$ for $i>\el$. Note that $f_0=f$, $f_D=f'$. With this, we write
\[
\|f-f'\|_{L_\infty}\leq\sum_{i=1}^D \|f_{i+1}-f_i\|_{L_\infty}.
\]
Observe that $f_{i-1}$ and $f_{i}$ have equal layers except the $i$th layer. Let $g_1$ be the function of the first $i-1$ layers and $g_2$ be the function of layers $i+1$ to $D$. We have that $f_{i+1}(\x)-f_i(\x)=g_2(\phi(\Kb_i(g_1(\x))))-g_2(\phi(\Kb_i'(g_1(\x))))$ where $\Kb_i,\Kb_i'$ differ in the $i$th layer kernels of $f$ and $f'$ created from $\beta_i\lah{i}$ and $\beta_i'\lah{i}$ respectively. Also, observe that $g_1$ is $B'^{i-1}$ Lipschitz and $g_2(\phi(\cdot))$ is $B'^{D-i}$ Lipschitz functions. Hence,
\begin{align}
|g_2(\phi(\Kb_i(g_1(\x))))-g_2(\phi(\Kb_i'(g_1(\x))))|\leq B'^{D-i}|\Kb_i(g_1(\x))-\Kb_i'(g_1(\x))|\leq B'^{D-i}\delta B'^{i-1}\leq \delta B'^{D-1}.
\end{align}
Summing over all $i$, this implies that $\|f-f'\|_{L_\infty}\leq D\delta B'^{D-1}$. Recalling $|\Fc'|\leq (2B'/\delta)^D$ and setting $\delta=\eps/(DB'^{D-1})$, the $\eps$ covering number of $\Fc_{\hat{\lan},B}$ is $N_\eps\leq (2DB'^D/\eps)^D=(2DB/\eps)^D$ which implies $\log N_\eps=D\log (\frac{2DB}{\eps})$. Now, since all kernels have Euclidean norm bounded by $B'$, we have $\|\func{}\|_{L_\infty}\leq B$ and $\|f\|_{L_\infty}\leq B$ for all $f\in \Fc$. This also implies $\tsub{\zm{\func{\x}}}=\order{B}$. Hence, we can apply Theorem \ref{cerm thm} to conclude the proof of Theorem \ref{ssa}.% in the regime $\eps\leq $
%Pick a $\delta$ cover $\Cc'$ of the interval $[0,R]$ of size $R/\delta$. Create a cover $\Cc$ by picking all scalings from $\Cc'$ and all $2^D$ sign configurations for $\pm\lah{\el}$. $|\Cc|\leq 2^DR/\delta$. For any $f\in \Fc$, pick $f'\in \Cc$ such that $f'$ has the same sign pattern and its scaling is $\delta $ close. Then, $f'-f$ is just a scaled version of $f$ and $\delta \lipp$ Lipschitz. Hence, $\Cc$ is a $\delta$ cover of $\Fc$ with respect to $\|\cdot\|_{L_\infty}$. Next, observe that $\func{\cdot}$ is $\lipp$ Lipschitz. Now, applying Theorem \ref{cerm thm}, we find that if $n\geq \order{\max\{\eps^{-1},\eps^{-2}\}(\lipp+R)^2D\log (\lipp Rp/\eps)}$, \eqref{sign scale eq} holds.
\section*{Acknowledgements}
M.S. is supported by the Air Force Office of Scientific Research Young Investigator Program (AFOSR-YIP) under award number FA9550-18-1-0078, a Google Faculty Research Award and a grant by the Northrop Grumman Cybersecurity Research Consortium.

\small{
\bibliographystyle{plain}
\bibliography{Bibfiles}
}

%\appendix
\end{document}

%% file: notation.tex
\usepackage{hyperref,graphicx,amsmath,amsfonts,amssymb,bm,url,breakurl,epsfig,epsf,color,fullpage,MnSymbol,mathbbol,fmtcount,semtrans,cite,caption,subcaption,multirow,comment}
\usepackage[noend]{algpseudocode}

\usepackage{amssymb}

\usepackage[utf8]{inputenc} % allow utf-8 input
\usepackage[T1]{fontenc}    % use 8-bit T1 fonts
\usepackage{url}            % simple URL typesetting
\usepackage{booktabs}       % professional-quality tables
\usepackage{amsfonts}       % blackboard math symbols
\usepackage{nicefrac}       % compact symbols for 1/2, etc.
\usepackage{microtype}      % microtypography

\makeatletter
\providecommand*{\boxast}{%
  \mathbin{% as \boxplus and \boxtimes
    \mathpalette\@boxit{*}%
  }%
}
\newcommand*{\@boxit}[2]{%
  \sbox0{$\m@th#1\Box$}%
  \ifx#1\displaystyle \ht0=\dimexpr\ht0+.05ex\relax \fi
  \ifx#1\textstyle \ht0=\dimexpr\ht0+.05ex\relax \fi
  \ifx#1\scriptstyle \ht0=\dimexpr\ht0+.04ex\relax \fi
  \ifx#1\scriptscriptstyle \ht0=\dimexpr\ht0+.065ex\relax \fi
  \sbox2{$#1\vcenter{}$}% \ht2 is positionn of math axis
  \rlap{%
    \hbox to \wd0{%
      \hfill
      \raisebox{%
        \dimexpr.5\dimexpr\ht0+\dp0\relax-\ht2\relax
      }{$\m@th#1#2$}%
      \hfill
    }%
  }%
  \Box
}
\makeatother

  \makeatletter
\def\BState{\State\hskip-\ALG@thistlm}
\makeatother

  \usepackage{mathtools}

\usepackage{titlesec}

\usepackage{tikz}
\usepackage{pgfplots}
\usetikzlibrary{pgfplots.groupplots}

\titleformat{\paragraph}
{\normalfont\normalsize\bfseries}{\theparagraph}{1em}{}
\titlespacing*{\paragraph}
{0pt}{3.25ex plus 1ex minus .2ex}{1.5ex plus .2ex}

\usepackage{movie15}

\usepackage{cite}
\usepackage{caption}
\usepackage[bottom,hang,flushmargin]{footmisc} 

\newcommand{\tsn}[1]{{\left\vert\kern-0.25ex\left\vert\kern-0.25ex\left\vert #1 
    \right\vert\kern-0.25ex\right\vert\kern-0.25ex\right\vert}}

\definecolor{darkred}{RGB}{150,0,0}
\definecolor{darkgreen}{RGB}{0,150,0}
\definecolor{darkblue}{RGB}{0,0,200}
\hypersetup{colorlinks=true, linkcolor=darkred, citecolor=darkgreen, urlcolor=darkblue}

\newtheorem{theorem}{Theorem}[section]

\newtheorem{lemma}[theorem]{Lemma}

\newtheorem{definition}[theorem]{Definition}

\newcommand{\eps}{\varepsilon}

\newcommand{\KB}{\bar{K}}

\newcommand{\lay}[1]{{\vct{k}}^{(#1)}}
\newcommand{\lah}[1]{{\vct{\hat{k}}}^{(#1)}}
\newcommand{\lan}{\vct{k}}
\newcommand{\rest}{{\bf{r}}}

\newcommand{\fcp}{\vct{g}^{\text{CNN}}}
\newcommand{\beq}{\begin{equation}}

\newcommand{\eeq}{\end{equation}}

\newcommand{\modu}{{{\text{{mod}}}}}

\newcommand{\func}[1]{{f_{\text{CNN}}}(#1)}
\newcommand{\funh}[1]{{{\hat{f}}_{\text{CNN}}}(#1)}

\newcommand{\nn}{\nonumber}

\newcommand{\bt}{\bigotimes}
\newcommand{\robt}[1]{\bigotimes_{\ell=1}^D{#1}_{\ell}}
\newcommand{\robtu}[1]{\bigotimes_{\ell=1}^D{#1}^{(\ell)}}

\newcommand{\ten}{(\{\vb_\ell\}_{\ell=1}^D)}
\newcommand{\uten}{(\{\ub_\ell\}_{\ell=1}^D)}

\newcommand{\yavg}{{{y}_{\text{avg}}}}
\newcommand{\ravg}[1]{{{r}_{{\text{avg}}}(#1)}}

\newcommand{\Lc}{{\cal{L}}}
\newcommand{\btd}{{\bt_{\ell=1}^D d_\ell}}
\newcommand{\Ro}{{\cal{RO}}}

\newcommand{\Tb}{{\mtx{T}}}

\newcommand{\Eb}{{\mtx{E}}}
\newcommand{\Hb}{{\mtx{H}}}

\newcommand{\sigmap}{\phi'}

\newcommand{\Iden}{{\mtx{I}}}

\newcommand{\order}[1]{{\cal{O}}(#1)}

\newcommand{\el}{{\ell}}
\newcommand{\tn}[1]{\|{#1}\|_{\ell_2}}

\newcommand{\tf}[1]{\|{#1}\|_{F}}
\newcommand{\te}[1]{\|{#1}\|_{\psi_1}}
\newcommand{\tsub}[1]{\|{#1}\|_{\psi_2}}

\newcommand{\tsup}[1]{\|{#1}\|_{\psi_a}}

\newcommand{\Cc}{\mathcal{C}}

\newcommand{\Bc}{\mathcal{B}}

\newcommand{\pa}{{\partial}}
\newcommand{\Nn}{\mathcal{N}}

\newcommand{\vb}{\vct{v}}

\newcommand{\fb}{\vct{f}}
\newcommand{\ib}{{\bf{i}}}

\newcommand{\Ic}{{\mathcal{I}}}

\newcommand{\li}{\left<}
\newcommand{\ri}{\right>}

\newcommand{\ub}{{\vct{u}}}

\newcommand{\h}{\vct{h}}
\newcommand{\hh}{{\vct{h}}}

\newcommand{\zm}[1]{{\text{{\bf{zm}}}(#1)}}

\newcommand{\Fc}{\mathcal{F}}

\newcommand{\cgain}{\alpha_{\text{CNN}}}

\newcommand{\xbr}{\bar{\h}}

\newcommand{\mapp}{\mu}

\newcommand{\lcnn}{{\mtx{L}}_{\text{CNN}}}

\newcommand{\smo}{S}

\newcommand{\fronorm}[1]{\left\|#1\right\|_{F}}

\newcommand{\twonorm}[1]{\left\|#1\right\|_{\ell_2}}

\newcommand{\infnorm}[1]{\left\|#1\right\|_{\ell_\infty}}

\newcommand{\abs}[1]{\left|#1\right|}

\newcommand{\lipp}{\prod_{\ell=1}^D \twonorm{\vct{k}^{(\ell)}}}

\newcommand{\nntd}{\text{DeepTD}}
\newcommand{\x}{\vct{x}}

\newcommand{\y}{\vct{y}}

\newcommand{\bgl}{{~\big |~}}

\definecolor{emmanuel}{RGB}{255,127,0}

\newcommand{\Kb}{{\mtx{K}}}

\newcommand{\R}{\mathbb{R}}
\newcommand{\Pro}{\mathbb{P}}

\newcommand{\E}{\operatorname{\mathbb{E}}}

\newcommand{\vct}[1]{\bm{#1}}
\newcommand{\mtx}[1]{\bm{#1}}

\newcommand{\X}{{\mtx{X}}}
\newcommand{\Y}{{\mtx{Y}}}

\numberwithin{equation}{section} 

\def \endprf{\hfill {\vrule height6pt width6pt depth0pt}\medskip}
\newenvironment{proof}{\noindent {\bf Proof} }{\endprf\par}

%% file: intro.tex
\section{Introduction}
Deep neural network (DNN) architectures have led to state of the art performance in many domains including image recognition, natural language processing, recommendation systems, and video analysis \cite{he2016deep,krizhevsky2012imagenet,van2013deep,collobert2008unified,KJcell}. Convolutional neural networks (CNNs) are a class of deep, feed-forward neural networks with a specialized DNN architecture. CNNs are responsible for some of the most significant performance gains of DNN architectures. In particular, CNN architectures have led to striking performance improvements for image/object recognition tasks. Convolutional neural networks, loosely inspired by the visual cortex of animals, construct increasingly higher level features (such as mouth and nose) from lower level features such as pixels. An added advantage of CNNs which makes them extremely attractive for large-scale applications is their remarkable efficiency which can be attributed to: (1) intelligent utilization of parameters via weight-sharing, (2) their convolutional nature which exploits the local spatial structure of images/videos effectively, and (3) highly efficient matrix/vector multiplication involved in CNNs compared to fully-connected neural network architectures.

Despite the wide empirical success of CNNs the reasons for the effectiveness of neural networks and CNNs in particular is still a mystery. Recently there has been a surge of interest in developing more rigorous foundations for neural networks \cite{sol2017,zhong2017recovery,brutzkus2017globally, soltanolkotabi2017learning, oymak2018learning,zhong2017learning,janzamin2015beating,li2017convergence,mei2018mean}. Most of this existing literature however focus on learning shallow neural networks typically consisting of zero or one hidden layer. In practical applications, depth seems to play a crucial role in constructing progressively higher-level features from pixels. Indeed, state of the art Resnet models typically have hundreds of layers. Furthermore, recent results suggest that increasing depth may substantially boost the expressive power of neural networks \cite{raghu2016expressive,cohen2016expressive}.

In this paper, we propose an algorithm for approximately learning an arbitrarily deep CNN model with rigorous guarantees. Our goal is to provide theoretical insights towards better understanding when training deep CNN architectures is computationally tractable and how much data is required for successful training. We focus on a realizable model where the inputs are chosen i.i.d.~from a Gaussian distribution and the labels are generated according to planted convolutional kernels. We use both labels and features in the training data to construct a tensor. Our first insight is that, in the limit of infinite data this tensor converges to a \emph{population} tensor which is approximately rank one and whose factors reveal the direction of the kernels. Our second insight is that even with finite data this \emph{empirical} tensor is still approximately rank one. We show that the gap between the population and empirical tensors provably decreases with the increase in the size of the training data set and becomes negligible as soon as the size of the training data becomes proportional to the total numbers of the parameters in the planted CNN model. Combining these insights we provide a tensor decomposition algorithm to learn the kernels from training data. We show that our algorithm approximately learns the kernels (up to sign/scale ambiguities) as soon as the size of the training data is proportional to the total number of parameters of the planted CNN model. Our results can be viewed as a first step towards provable end-to-end learning of practical deep CNN models. Extending the connections between neural networks and tensors \cite{janzamin2015beating,cohen2016expressive,zhong2017learning}, we show how tensor decomposition can be utilized to approximately learn deep networks despite the presence of nonlinearities and growing depth. While our focus in this work is limited to tensors, we believe that our proposed algorithm may provide valuable insights for initializing local search methods (such as stochastic gradient descent) to enhance the quality and/or speed of CNN training.

%% file: cerm.tex
\section{Resolving sign and scaling ambiguities} \label{resolve}
% via centered ERM
%an algorithm which aims to resolve sign and scale ambiguities when possible. 
We note that DeepTD operates by accurately approximating the rank one tensor $\bt_{\el=1}^D\lay{\el}$ from data. Therefore, DeepTD can only recover the convolutional kernels up to Sign/Scale Ambiguities (SSA). In general, it may not be possible to recover the ground truth kernels from the training data. For instance, when activations are ReLU, the norms of the kernels cannot be estimated from data as multiplying a kernel and dividing another by the same positive scalar leads to the same training data. However, we can try to learn a good approximation $\funh{}$ of the network $\func{}$ to minimize the risk $\E[(\func{\x}-\funh{\x})^2]$.

%Empirical Risk Minimization (ERM) %an approach called 
To this aim, we introduce Centered Empirical Risk Minimization (CERM) which is a slight modification of Empirical Risk Minimization (ERM). Let us first describe how finding a good $\funh{}$ can be formulated with CERM. Given $n$ i.i.d.~data points $\{(\x_i,y_i)\}_{i=1}^n\sim(\x,y)$, and a function class $\Fc$, CERM applies ERM after centering the residuals. Given $f\in\Fc$, define the average residual function $\ravg{f}=\frac{1}{n}\sum_{i=1}^ny_i-f(\x_i)$. We define the Centered Empirical Risk Minimizer as
\begin{align}\label{cerm}
\hat{f}&=\min_{f\in \Fc} \frac{1}{n}\sum_{i=1}^n (y_i-f(\x_i)-\ravg{f})^2,\nonumber\\
&=\min_{f\in \Fc} \frac{1}{n}\sum_{i=1}^n \left(y_i-f(\x_i)-\E[(y_i-f(\x_i))]\right)^2-\frac{1}{n^2}\left(\sum_{i=1}^n y_i-f(\x_i)-\E[(y_i-f(\x_i))]\right)^2.
\end{align}
The remarkable benefit of CERM over ERM is the fact that, the learning rate doesn't suffer from the label or function bias. This is in similar nature to the \nntd~algorithm that applies label centering. In the proofs (in particular Section \ref{sec cerm}, Theorem \ref{cerm thm}) we provide a generalization bound on the CERM solution \eqref{cerm} in terms of the Lipschitz covering number of the function space. While \eqref{cerm} can be used to learn all kernels, it does not provide an efficient algorithm. Instead, we will use CERM to resolve SSA after estimating the kernels via \nntd. Interestingly, this approach only requires a few ($\order{D}$) extra training samples. Inspired from CERM, in Section \ref{greedy sec}, we propose a greedy algorithm to address SSA. We will apply CERM to the following function class with bounded kernels,
\begin{align}
\small{\Fc_{\hat{\lan},B}:=\big\{f:\R^p\mapsto\R\text{ }\big|\text{ }\text{f is a CNN function of the form } \eqref{CNNform} \text{ with kernels }\{\beta_\ell\hat{\vct{k}}^{(\ell)}\}}_{\ell=1}^D\text{ with }|{\beta}_\ell|\le B^{\frac{1}{D}}\big\}.\label{cnn form 1}
\end{align}
In words this is the function class of all CNN functions with kernels the same as those obtained by DeepTD up to sign/scale ambiguities $\{\beta_\ell\}_{\ell=1}^D$ where the maximum scale ambiguity is $B$.
\begin{theorem} \label{ssa}Let $\func{}$ be defined via \eqref{CNNform} with convolutional kernels $\{\vct{k}^{(\ell)}\}_{\ell=1}^D$ obeying $\tn{\lay{\el}}\leq B^{1/D}$ for some $B>0$ and consider the function class $\Fc_{\hat{\lan},B}$ above with the same choice of $B$. Assume we have $n$ i.i.d. samples $(\x_i,y_i)\sim (\x,y)$ where $\x\sim\Nn(0,\Iden_p)$ and $y=\func{\x}$. Suppose for some $\eps\leq B$,
\begin{align*}
n\ge c B^2D \log\left(\frac{CDB p}{\epsilon}\right)\max\left(\frac{1}{\epsilon},\frac{1}{\epsilon^2}\right),
\end{align*}
holds for fixed numerical constants $c,C>0$. Then the solution $\hat{f}$ to the CERM problem \eqref{cerm} obeys
\begin{align}
\E\Big[\left(\hat{f}(\x)-\func{\x}\right)^2\Big]\leq \min_{f\in \Fc_{\hat{\lan},B}}\E[(f(\x)-\func{\x})^2]+\eps.\label{sign scale eq}
\end{align}
on a new sample $\x\in\Nn(0,\Iden_p)$ with probability at least $1-e^{-\gamma n}-4n\exp(-p)$ with $\gamma>0$ an absolute constant.
\end{theorem}
The above theorem states that CERM finds the sign/scale ambiguity that accurate estimates the labels on new data as long as the number of samples which are used in CERM exceeds the depth of the network by constant/log factors. In the next section we present a greedy heuristic for finding the CERM estimate.

\subsection{Greedy algorithm for resolving sign and scale ambiguities}\label{greedy sec}

In order to resolve SSA, inspired from CERM, we propose Algorithm \ref{algo 1} which operates over the function class,
\begin{align}
\small{\Fc_{\hat{\lan}}:=\big\{\gamma f:\R^p\mapsto\R\text{ }\big|\text{ }\text{f is a CNN of the form } \eqref{CNNform} \text{ with kernels }\{\beta_\ell\hat{\vct{k}}^{(\ell)}\}}_{\ell=1}^D\text{ with }{\beta}_\ell\in \{1,-1\},~\gamma\geq 0\big\}.\label{cnn form 2}
\end{align}%. In the first phase it determines the signs and later on it picks a global constant 
It first determines the signs $\beta_\el$ by locally optimizing the kernels and then finds a global scaling $\gamma>0$. In the first phase, the algorithm attempts to maximize the correlation between the centered labels $y_{c,i}=y_i-n^{-1}\sum_{i=1}^ny_i$ and the $\funh{}$ predictions given by $\hat{y}_{c,i}=\hat{y}_i-n^{-1}\sum_{i=1}^n\hat{y}_i$. It goes over all kernels one by one and it flips a kernel ($\lah{\el}\rightarrow -\lah{\el}$) if flipping increases the correlation. This process goes on as long as there is an improvement. Afterwards, we use a simple linear regression to get the best scaling $\gamma$ by minimizing the centered empirical loss $\sum_{i=1}^n (y_{c,i}-\gamma\hat{y}_{c,i})^2$. While our approach is applicable to arbitrary activations, it is tailored towards homogeneous activations ($\phi(cx)=c\phi(x)$). The reason is that for homogeneous activations, function classes \eqref{cnn form 1} and \eqref{cnn form 2} coincide and a single global scaling $\gamma$ is sufficient. Note that ReLU and the identity activation (i.e.~no activation) are both homegeneous, in fact they are elements of a larger homogeneous activation family named Leaky ReLU. Leaky ReLU is parametrized by some scalar $0\leq \beta\leq 1$ and defined as follows
\[
LReLU(x)=\begin{cases}x~\text{if}~x\geq 0,\\\beta x~\text{if}~x<0.\end{cases}
\]
%We remark that one can use alternative approaches to optimize over $\Fc_{\hat{\lan},B}$ of \eqref{cnn form 1}, such as line searching over $\beta_\el$'s and minimizing least-squares loss. In our experiments (with ReLU activation), 
As we shall discuss in our numerical experiments in Section \ref{numeric sec}, Algorithm \ref{algo 1} works well for different depths ($D=4,8,12$) and kernel widths $2\leq d_\el\leq 10$. We would like to point out however that it appears that this local search heuristic can indeed get stuck in a suboptimal local minima even for a depth $4$ convolutional network (see Figure \ref{fig perf} and the associated discussion). Developing efficient algorithms that can reliably construct a good $\funh{}$ from $\{\lah{\el}\}_{\el=1}^D$ is an interesting direction for future research.

%We also observed that along with  however it can 

\begin{algorithm} \caption{Greedily algorithm for resolving sign/scale ambiguities for Leaky ReLU activations.}\label{algo 1}
\begin{algorithmic}[1]
\Procedure{MaxCorr}{}
\item {\bf{Inputs:}} Data $(y_i,\x_i)_{i=1}^n$, estimates $\{\lah{\el}\}_{\el=1}^D$.
%\State 
\State $\rho_{\max}\gets|\text{Corr}(\{\lah{\el}\}_{\el=1}^D,0)|$,~FLIP$\gets${TRUE}.
{\While {FLIP}
\State FLIP$\gets${FALSE}.
{\For {$1\leq\el\leq D$}

\State $\rho\gets|\text{Corr}(\{\lah{1},~\dots,~-\lah{\el},~\dots,~\lah{D}\},0)|$.
{\If {$\rho>\rho_{\max}$}
\State $\rho_{\max}\gets\rho$
\State $\lah{\el}\gets-\lah{\el}$
\State FLIP$\gets$TRUE
\EndIf}
\EndFor}
\EndWhile}

\item \hspace{17pt}$\gamma\gets \text{Corr}(\{\lah{\el}\}_{\el=1}^D,1)$.\\
\Return kernels $\{\lah{\el}\}_{\el=1}^D$, scaling~$\gamma$.
\EndProcedure
\end{algorithmic}
\end{algorithm}
\begin{algorithm} \caption{Return the correlation between centered labels.}
\begin{algorithmic}[1]
\Procedure{Corr}{$\{\lah{\el}\}_{\el=1}^D$,~\text{opt}}
\State $\hat{y}_{i}\gets \func{\{\lah{\el}\}_{\el=1}^D;\x_i}$.
\State $y_{\text{c},i}\gets y_i-\frac{1}{n}\sum_{i=1}^ny_i$~~~\text{and}~~~$\hat{y}_{\text{c},i}\gets \hat{y}_i-\frac{1}{n}\sum_{i=1}^n\hat{y}_i$.
%\State 
\State $\rho\gets \sum_{i=1}^n {y}_{\text{c},i}\hat{y}_{\text{c},i}$.\\
\Return $\rho~\text{{\bf{if}} opt}=0,~\rho/(\sum_{i=1}^n \hat{y}^2_{\text{c},i})~\text{{\bf{if}} opt}=1$.
\EndProcedure
\end{algorithmic}
\end{algorithm}
%\begin{algorithm} \caption{Return the correlation between centered labels.}
%\begin{algorithmic}[1]
%\Procedure{Corr}{$\{\lah{\el}\}_{\el=1}^D$,~\text{opt}}
%\State $\hat{y}_{i}\gets \func{\{\lah{\el}\}_{\el=1}^D;\x_i}$.
%\State $y_{\text{c},i}\gets y_i-\frac{1}{n}\sum_{i=1}^ny_i$.
%\State $\hat{y}_{\text{c},i}\gets \hat{y}_i-\frac{1}{n}\sum_{i=1}^n\hat{y}_i$.
%\State $\rho\gets \sum_{i=1}^n {y}_{\text{c},i}\hat{y}_{\text{c},i}$.\\
%\Return $\rho~\text{{\bf{if}} opt}=0,~\rho/(\sum_{i=1}^n \hat{y}^2_{\text{c},i})~\text{{\bf{if}} opt}=1$.
%\EndProcedure
%\end{algorithmic}
%\end{algorithm}

%% file: numerical.tex
%Given a vector $\y$, its centering is defined The centerin

\section{Numerical experiments}\label{numeric sec}

\begin{figure}[t!]
 \begin{subfigure}[b]{0.5\textwidth}
        \includegraphics[width=\textwidth]{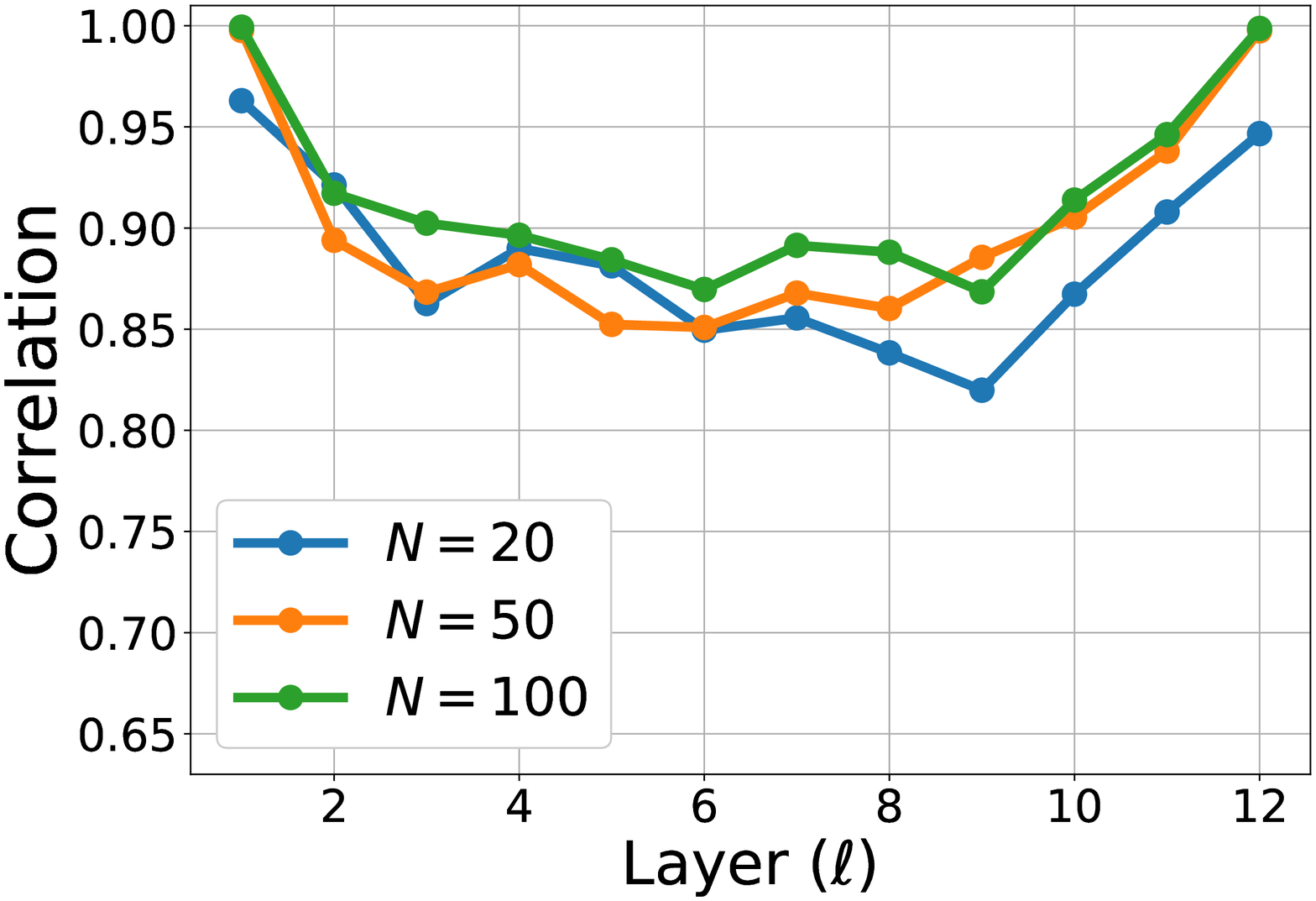}%triple.eps
        \caption{$d=2$, $D=12$.}
        \label{fig depth a}
    \end{subfigure} ~
    \begin{subfigure}[b]{0.5\textwidth}
        \includegraphics[width=\textwidth]{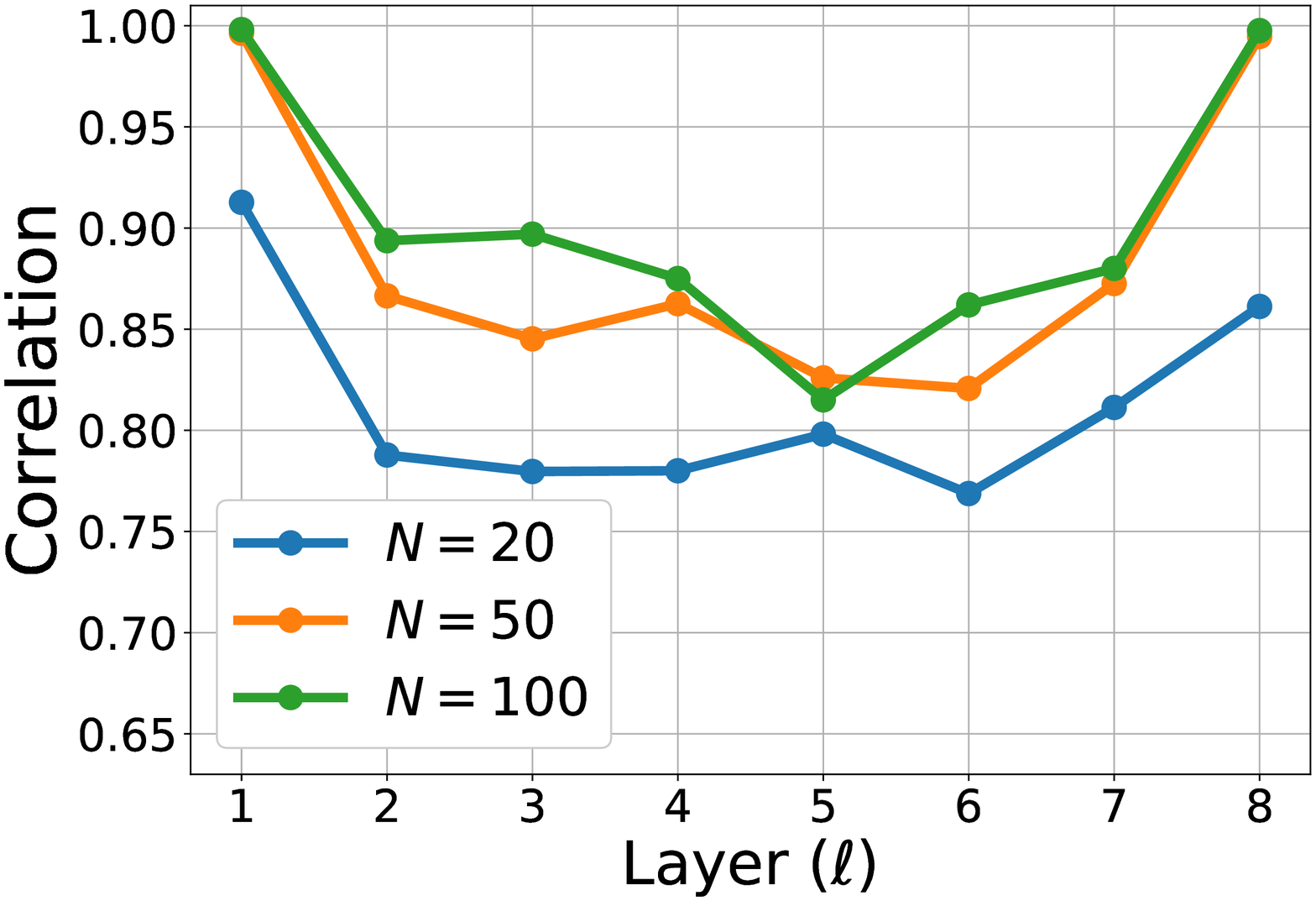}%triple.eps
        \caption{$d=3$, $D=8$.}
         \label{fig depth b}
    \end{subfigure}
    \caption{Correlations ($\text{corr}(\lah{\el},\lay{\el})=\abs{\li\lah{\el},\lay{\el}\ri}$) between the \nntd~estimate and the ground truth kernels for different layers and various over-sampling ratios ($N=\frac{n}{\sum_{\el=1}^Dd_\el}$).}
        \label{fig depth}
\end{figure}

\begin{table}
\begin{centering}
\begin{tabular}{c|c|c|c|c|c|cl}
\cline{2-7}
&  \multicolumn{3}{ c| }{$d=2,~D=12$}&\multicolumn{3}{ c| }{$d=3,~D=8$} \\ \cline{2-7}
&  \multicolumn{1}{ |c| }{$N=20$}&\multicolumn{1}{ c| }{$N=50$}&\multicolumn{1}{ c| }{$N=100$}&{$N=20$}&\multicolumn{1}{ c| }{$N=50$}&\multicolumn{1}{ c| }{$N=100$} \\ \cline{1-7}
%& & 2 & 3 & 5 & 7 \\ \cline{1-6}
\multicolumn{1}{ |c| }{\% Correct sign} & 0.83 & 0.95 & 0.93 & 0.65 & 0.87& \multicolumn{1}{ c| }{0.94}   \\ \cline{1-7}
%\multicolumn{1}{ |c  }{}                        &
\multicolumn{1}{ |c| }{Test loss (greedy)} & 0.53 & 0.45 & 0.43 & 0.62 & 0.48&   \multicolumn{1}{ c| }{0.42}   \\ \cline{1-7}
\multicolumn{1}{ |c| }{Test loss (oracle)} & 0.50 & 0.44 & 0.41 & 0.58& 0.44 & \multicolumn{1}{ c| }{0.40} \\ \cline{1-7 }
%\multicolumn{1}{ |c  }{}                        &
%\multicolumn{1}{ |c| }{lcm} & 3 & 3 & 1 & 1& a& \\ \cline{1-7}
\end{tabular}
\caption{Test performance of \nntd~estimates ($\text{Test MSE}=\E\Big[\left(\y-\hat{f}_{\text{CNN}}(\x)\right)^2\Big]/\E[\y^2]$). Greedy is the test loss when the signs of $\lah{\el}$ are determined by Algorithm \ref{algo 1}. Oracle is the test loss by picking the signs to ensure non-negative correlation with the ground truth kernels.}\label{table 1}
\end{centering}
\end{table}

\begin{figure}[t!]
 \begin{subfigure}[b]{0.5\textwidth}
        \includegraphics[width=\textwidth]{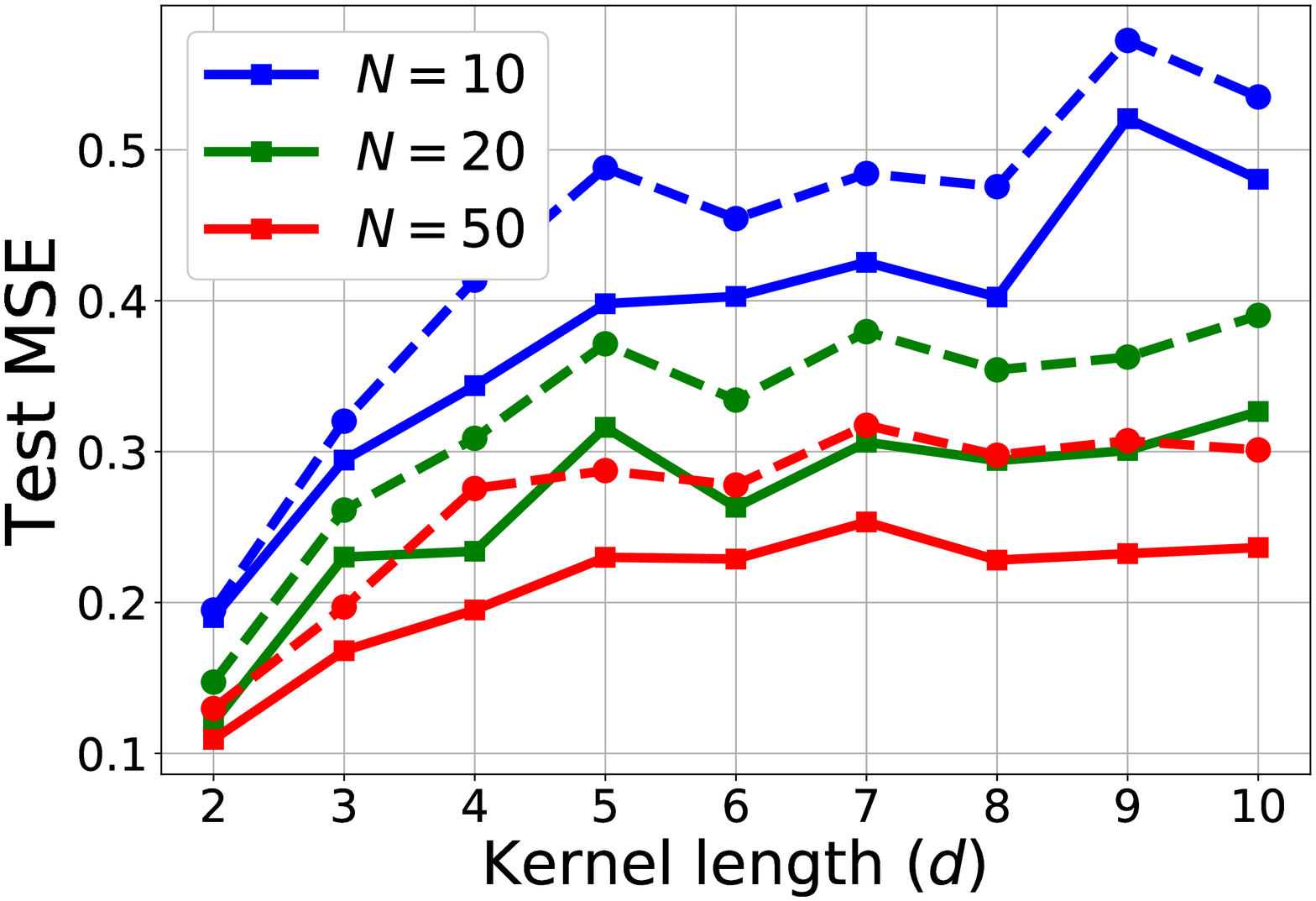}%triple.eps
        \caption{}
        \label{fig perf1}
    \end{subfigure} ~
    \begin{subfigure}[b]{0.5\textwidth}
        \includegraphics[width=\textwidth]{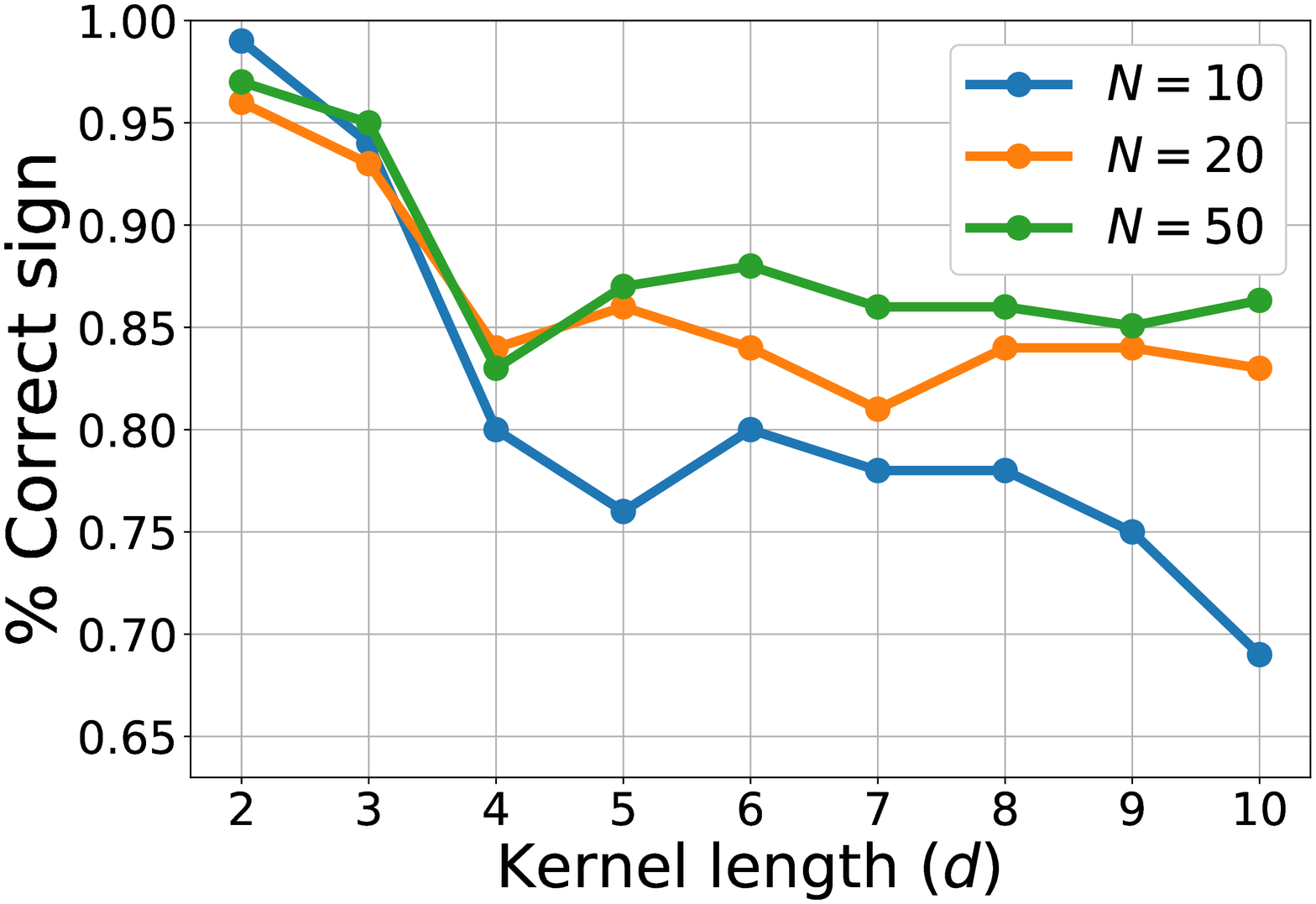}%triple.eps
        \caption{}
        \label{fig perf2}
    \end{subfigure}%Gaussian kernels Correlations between the \nntd~estimate and the ground truth kernels for different layers. Final activation is identity.
        \caption{a) The solid lines are the Test MSE ($\E\Big[\left(\y-\hat{f}_{\text{CNN}}(\x)\right)^2\Big]/\E[\y^2]$) error when using the correct kernel signs (oracle). The dashed lines are the Test MSE by determining the kernel signs with Algorithm \ref{algo 1}. b) Fraction of times Algorithm \ref{algo 1} successfully identified all signs.}
        \label{fig perf}
\end{figure}
%. Surprisingly, both approaches perform best for smaller values of $d$ (e.g.~$d=2,3$)

%The intention is ensuring that even the output kernel is nontrivial to estimate\footnote{With identity activation at the output layer, learning the output kernel boils down to a linear regression.}. 
Our goal in this section is to numerically corroborate the theoretical predictions of Section \ref{sec main res}. To this aim we use a CNN model of the form \eqref{CNNform} with $D$ layers and ReLU activations and set the kernel lengths to be all equal to each other i.e.~$d_4=\ldots=d_1=d$. We use the identity activation for the last layer (i.e.~$\phi_D(z)=z$) with the exception of the last experiment where we use a ReLU activation (i.e.~$\phi_D(z)=\max(0,z)$). We conducted our experiments in Python using the Tensorly library for the tensor decomposition in \nntd \cite{kossaifi2016tensorly}.  Each curve in every figure is obtained by averaging $100$ independent realizations of the same CNN learning procedure. Similar to our theory, we use Gaussian data points $\x$ and ground truth labels $\y=\func{\x}$.

We conduct two sets of experiments: The first set focuses on larger values of depth $D$ and the second set focuses on larger values of width $d$. In all experiments kernels are generated with random Gaussian entries and are normalized to have unit Euclidean norm. For the ReLU activation if one of the kernels have all negative entries, the output is trivially zero and learning is not feasible. To address this, we consider {\em{operational}} networks where at least $50\%$ of the training labels are nonzero. Here, the number $50\%$ is arbitrarily chosen and we verified that similar results hold for other values. To ensure the kernels do not have all negative entries and the network is operational we use a rejection sampling scheme. That is, if our generated network does not obey the $50\%$ assumption we discard that experiment and generate a new one. In practice, this is mainly an issue for small $d$ values only as the chance of an all-negative kernel is equal to $2^{-d}$. Finally, to study the effect of finite samples, we let the sample size grow proportional to the total degrees of freedom $\sum_{\el=1}^Dd_\el$. In particular, we set an oversampling factor $N=\frac{n}{\sum_{\el=1}^Dd_\el}$ and carry out the experiments for $N\in\{10,20,50,100\}$. While our theory requires $N\gtrsim\log D$, in our experiments, we typically observe that improvement is marginal after $N=50$.
% We also assume the number of samples $n$ exceed the total number of parameters in the model $\sum_{\ell=1}^D d_\ell$  by an oversampling factor $N$. That is,
%\[
%n=N\sum_{\ell=1}^Dd_\ell.
%\]
%for each trial and report the average over $500$ independent realizations for varying oversampling factors $N\in\{10,20,50\}$.
%This indicates that practical CNN models might benefit from tensor decomposition algorithms regardless of kernel width. 

In Figure \ref{fig depth}, we consider two networks with $d=2,D=12$ and $d=3,D=8$ configuration. We plot the absolute correlation between the ground truth and the estimates as a function of layer depth. For each hidden layer $1\leq \el\leq D$, our correlation measure ($y$-axis) is 
\[
\text{corr}(\lah{\el},\lay{\el})=\abs{\li\lah{\el},\lay{\el}\ri}.
\] 
This number is between $0$ and $1$ as the kernels and their estimates both have unit norm. We observe that for both $d=2$ and $d=3$, \nntd~consistently achieves correlation above 75\% for $N=20$. While our theory requires $d$ to scale quadratically with depth i.e.~$d\gtrsim D^2$, we find that even small $d$ values work well in our experiments. The effect of sample size becomes evident by comparing $N=20$ and $N=50$ for the input and output layers ($\ell=1,\ell=D$). In this case $N=50$ achieves perfect correlation. Interestingly, correlation values are smallest in the middle layers. In fact this even holds when $N$ is large suggesting that the rank one approximation of the population tensor provides worst estimates for the middle layers.

Table \ref{table 1} shows the test performance of \nntd~using an oracle and Algorithm \ref{algo 1}. Our test loss is defined as the normalized regression loss
\begin{align}
\text{Test MSE}=\frac{\E\Big[\left(\y-\hat{f}_{\text{CNN}}(\x)\right)^2\Big]}{\E[\y^2]},
\end{align}
where $\hat{f}_{\text{CNN}}$ is our estimate of $\func{}$. Here, the oracle knows the correct signs (i.e.~guaranteed to have $\li\lah{\el},\lay{\el}\ri\geq 0$ for all $1\leq \el\leq D$) but the scale is found by solving a regression problem in a similar fashion to Line $12$ of Algorithm \ref{algo 1}. We observe that for $N\geq 50$, signs of all kernels are correctly identified with $\gtrsim 90\%$ accuracy. This table also demonstrates that there is a minor difference in test loss between the oracle and our proposed greedy algorithm. We would like to note that a test loss of $\approx 40\%$ may appear large, especially as we achieve high kernel correlations $\gtrsim 85\%$ (for $N=100$). We believe that this behavior is due to the depth of the network which aggregates the small estimation errors of individual layers and results in a larger error in the overall function estimate $\hat{f}_{\text{CNN}}$.

\begin{figure}[t!]
 \begin{subfigure}[b]{0.5\textwidth}
        \includegraphics[width=\textwidth]{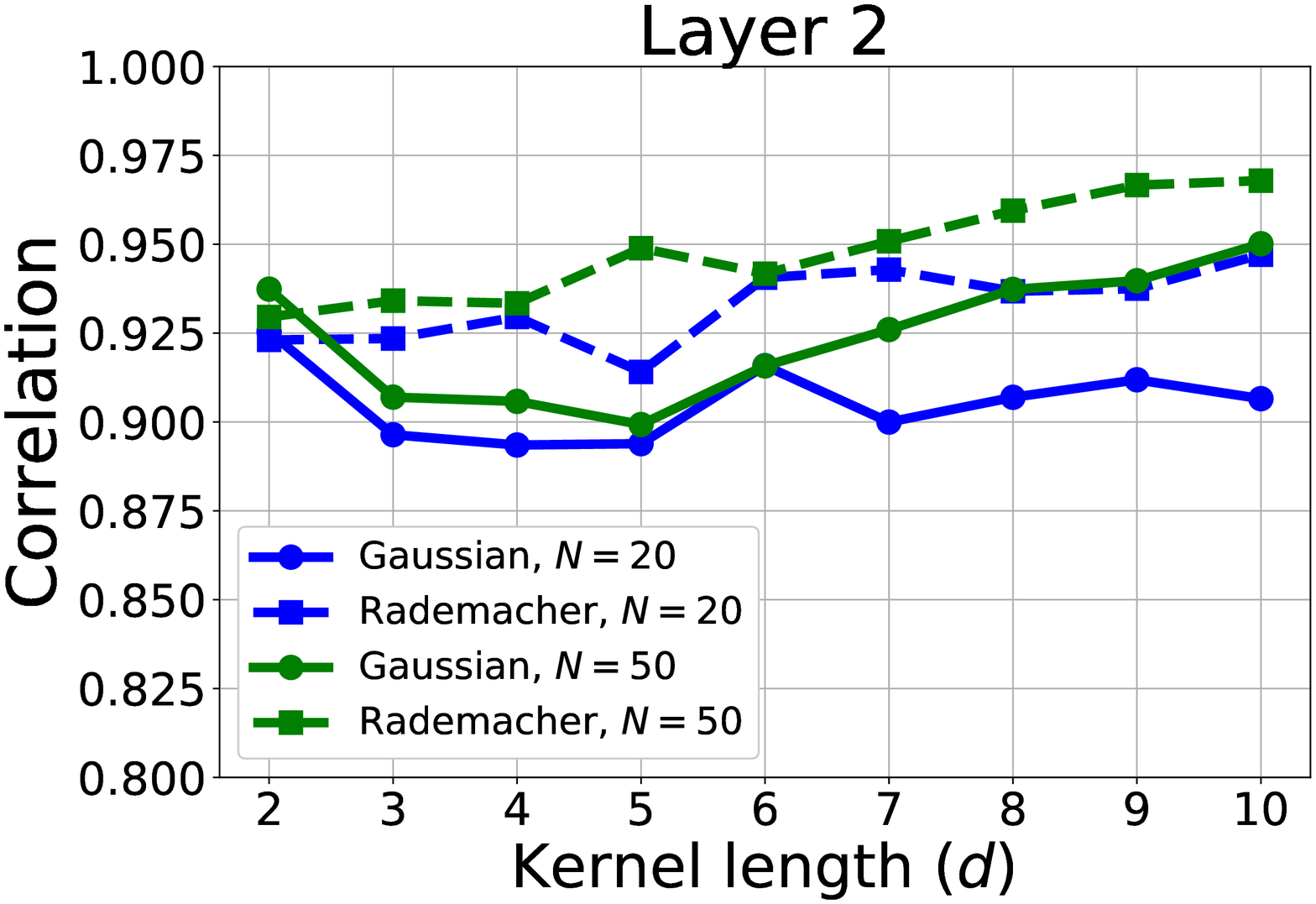}%triple.eps
        %\caption{}
    \end{subfigure} ~
    \begin{subfigure}[b]{0.5\textwidth}
        \includegraphics[width=\textwidth]{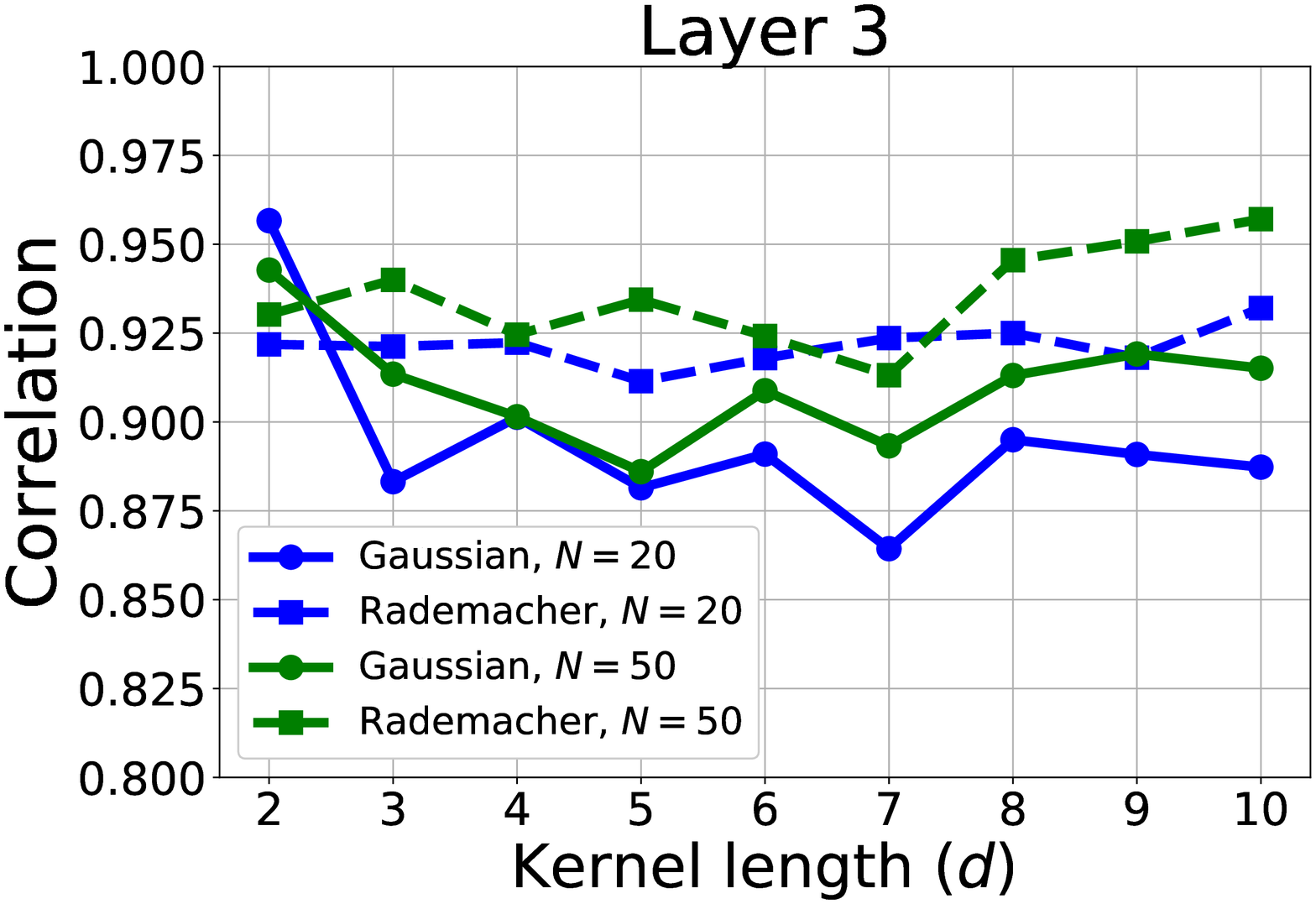}%triple.eps
        %\caption{}
    \end{subfigure}%Gaussian kernels Correlations between the \nntd~estimate and the ground truth kernels for different layers. Final activation is identity.
        \caption{Correlations ($\text{corr}(\lah{\el},\lay{\el})=\abs{\li\lah{\el},\lay{\el}\ri}$) when the ground truth kernels are Gaussian vs Rademacher. Gaussian is more spiky and results in lower correlation.}
        \label{fig comp1}
\end{figure}

%gets slightly worse after $d\geq 4$ and 
For the remaining experiments, we set $D=4$ and vary the width $d$ from $2$ to $10$. $N$ is varied between $10$ to $50$. Figure \ref{fig perf} summarizes the performance of the overall algorithm which is composed of \nntd~and Algorithm \ref{algo 1}. In Figure \ref{fig perf1}, we observe that small $d$ values $d=2,3$ achieve the best test performance. Performance appears to saturate after $d\geq 7$ both for greedy algorithm (dashed lines) and oracle (solid lines) which has the correct sign information. The greedy algorithm is consistently worse than the oracle which indicates that it may have reached a local minima (this claim is verified on several examples where the oracle and greedy approaches have differing signs). While signs are often correctly identified in Figure \ref{fig perf2}, the greedy algorithm fails in $15\%$ of the instances even for large $N$ which indicates the need for better heuristics heuristics to resolve sign and scale ambiguities.%that SSA has a nontrivial landscape.

Next we study our theoretical predictions regarding \nntd. In particular, Figure \ref{fig comp1} assesses the impact of diffuseness by plotting correlations for Layers $2$ and $3$. This is done by contrasting different kernel generation models, namely, kernels with i.i.d.~standard normal entries and Rademacher ($\pm1$) entries (normalized to have unit Euclidean norm). Rademacher kernels are strictly more diffused than Gaussian kernels as all entries have the same magnitude by construction. We observe that Rademacher kernels (dashed lines) achieve consistently higher correlations which is supported by our theory.

In Figure \ref{fig comp2}, we use a ReLU activation in the final layer and assess the impact of the centering procedure of the \nntd~algorithm which is a major theme throughout the paper. We define the NaiveTD algorithm which solves \eqref{deeptdalg} without centering in the empirical tensor i.e.
\begin{align}
\label{NaiveTD}
\hat{\vct{k}}^{(1)},\ldots,\hat{\vct{k}}^{(D)}=\underset{\vct{v}_1\in\R^{d_1},\vct{v}_2\in\R^{d_2},\ldots,\vct{v}_D\in\R^{d_D}}{\arg\max} \li\frac{1}{n}\sum_{i=1}^ny_i\mtx{X}_i,\robt{\vb}\ri\quad\text{subject to}\quad \twonorm{\vct{v}_1}=\ldots=\twonorm{\vct{v}_D}=1.
\end{align}
Since the activation of the final layer is ReLU, the output has a clear positive bias in expectation which will help demonstrating the importance of centering. We find that for smaller oversampling factors of $N=10$ or $N=20$, \nntd~has a visibly better performance compared with NaiveTD. The correlation difference is persistent among different layers (we plotted only Layers $1$ and $2$) and appears to grow with increase in the kernel size $d$.

Finally, in Figure \ref{fig comp3}, we assess the impact of activation nonlinearity by comparing the ReLU and identity activations in the final layer. We plot the first and final layer correlations for this setup. While the correlation performances of the first layer are essentially identical, the ReLU activation (dashed lines) achieves significantly lower correlation at the final layer. This is not surprising as the final layer passes through an additional nonlinearity.

\begin{figure}[t!]
 \begin{subfigure}[b]{0.5\textwidth}
        \includegraphics[width=\textwidth]{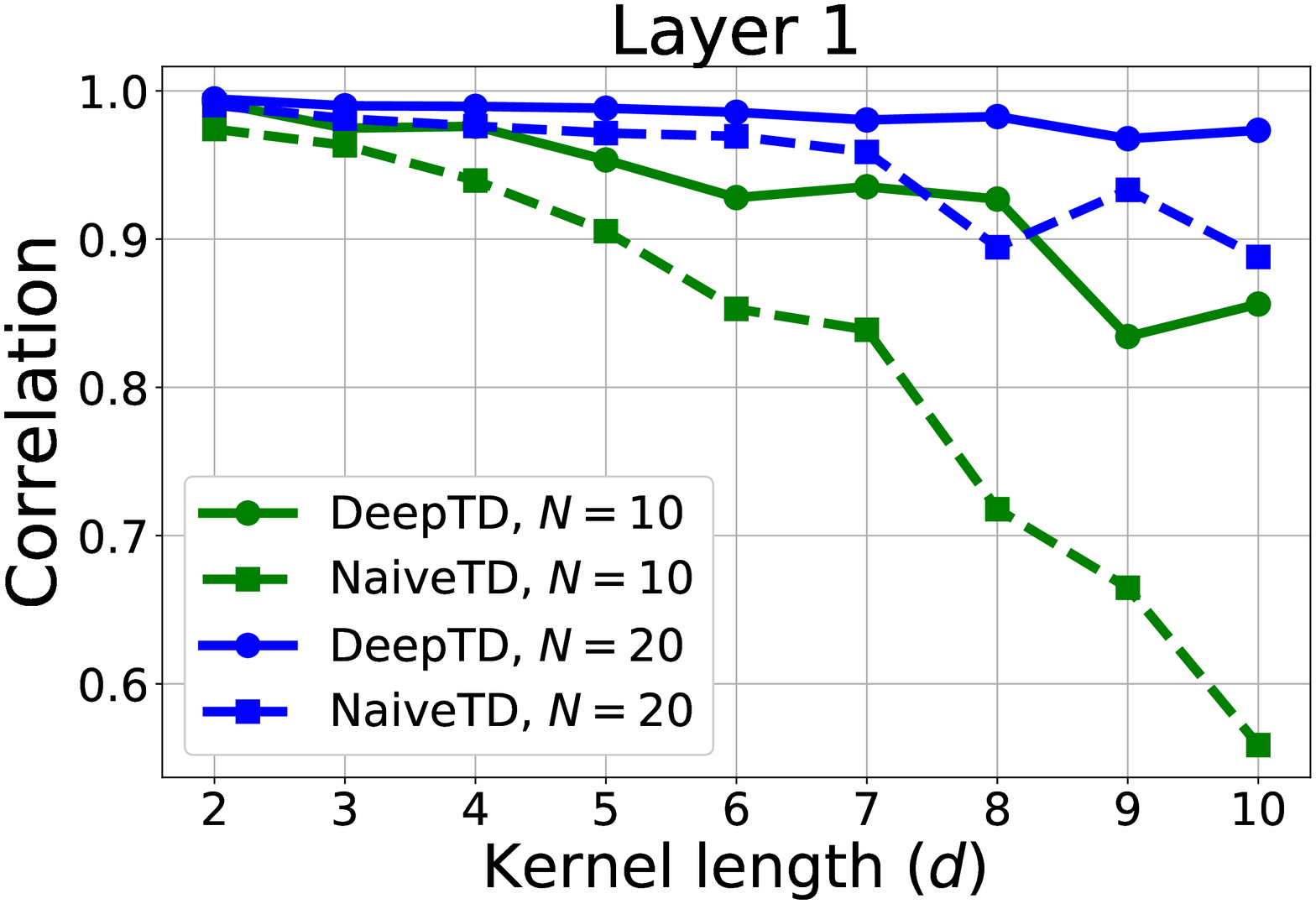}%triple.eps
        %\caption{}
    \end{subfigure} ~
    \begin{subfigure}[b]{0.5\textwidth}
        \includegraphics[width=\textwidth]{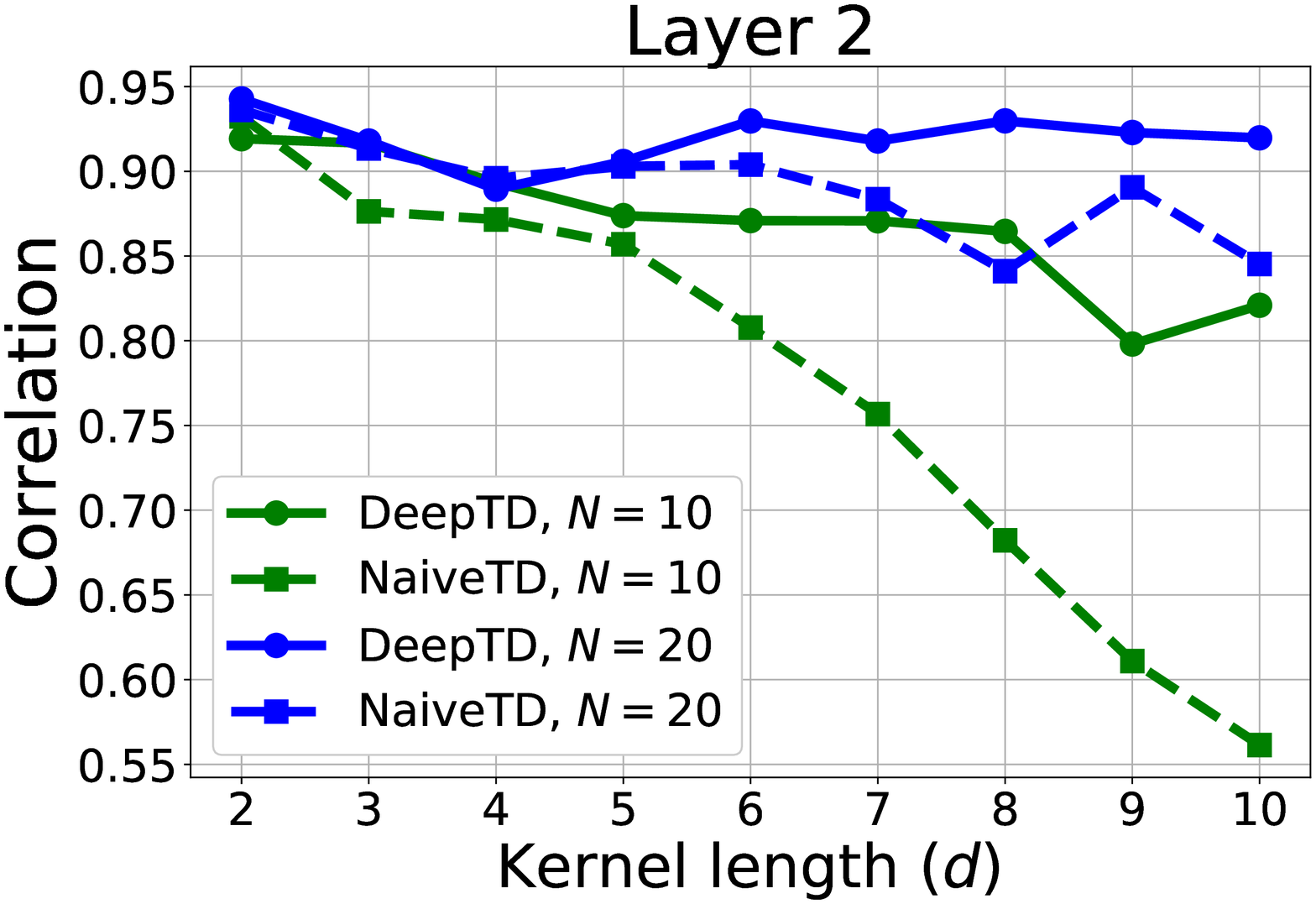}%triple.eps
        %\caption{}
    \end{subfigure}
        \caption{\nntd~estimate vs NaiveTD estimate when final activation is ReLU. Bias of NaiveTD results in significantly worse performance.}
        \label{fig comp2}
\end{figure}

\begin{figure}[t!]
 \begin{subfigure}[b]{0.5\textwidth}
        \includegraphics[width=\textwidth]{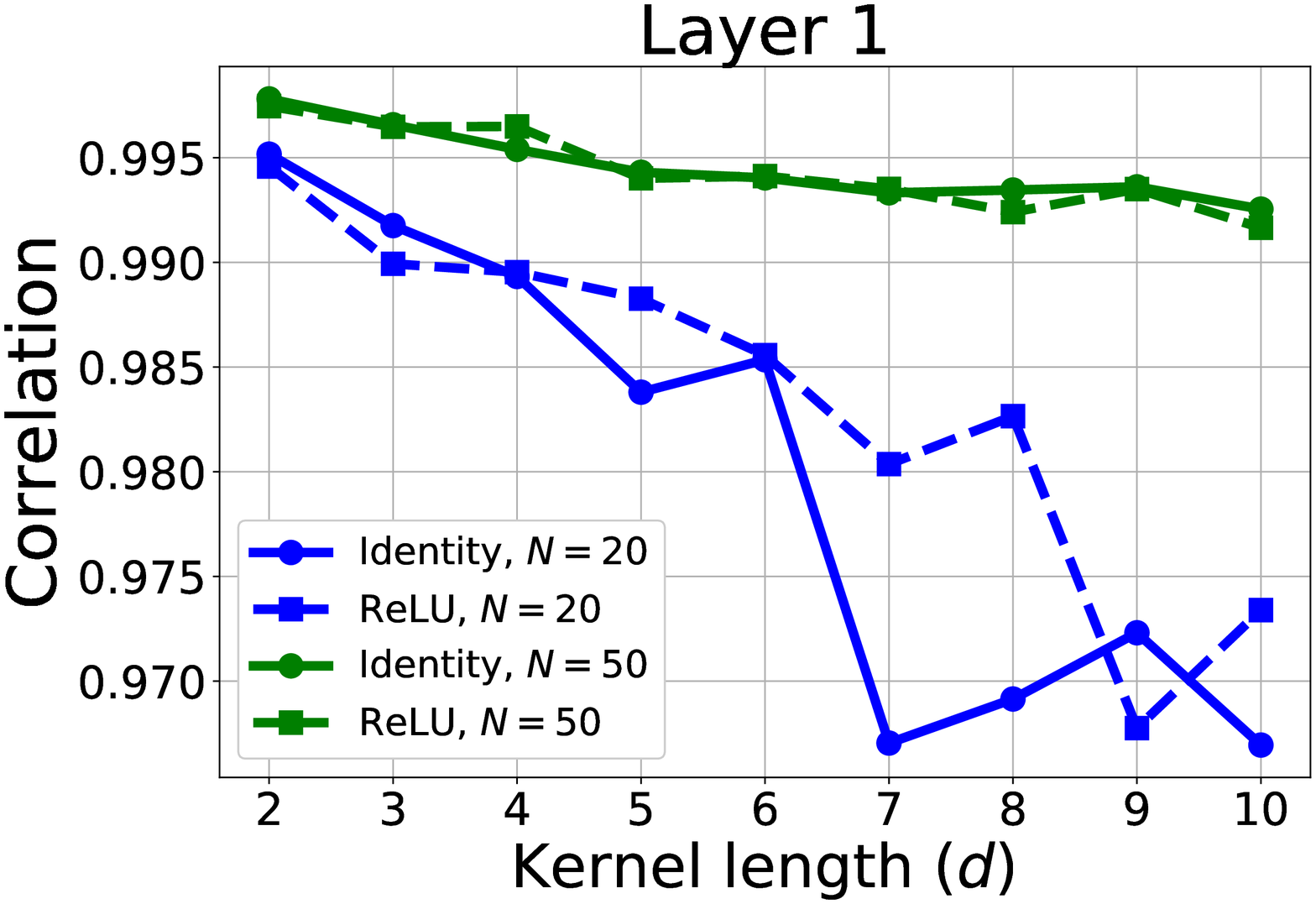}%triple.eps
        %\caption{}
    \end{subfigure} ~
    \begin{subfigure}[b]{0.5\textwidth}
        \includegraphics[width=\textwidth]{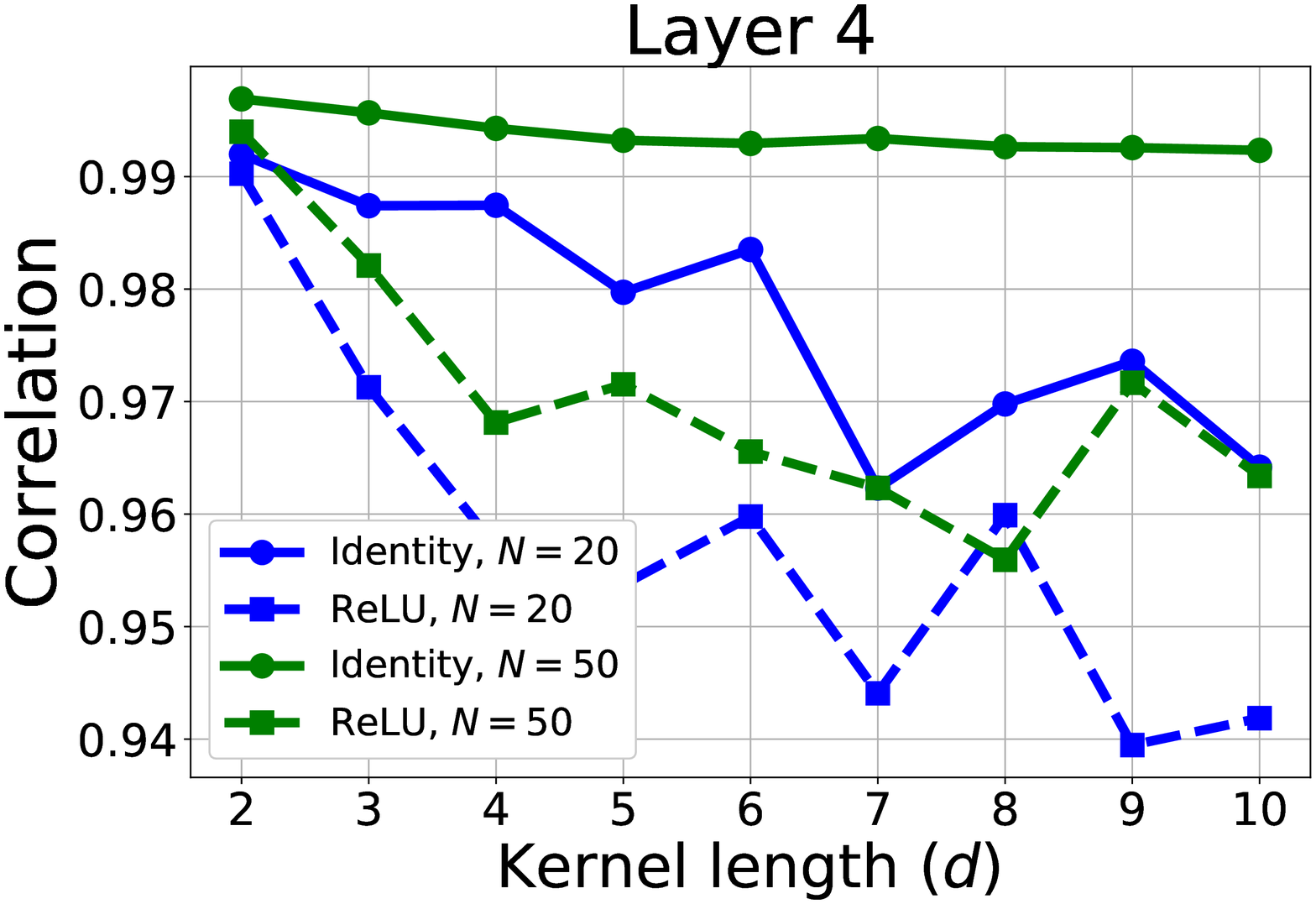}%triple.eps
        %\caption{}
    \end{subfigure}
        \caption{Comparison of performance of DeepTD when the final activation is ReLU in lieu of the identity activation. }
        \label{fig comp3}
\end{figure}